\documentclass[twoside]{article}

\usepackage[accepted]{aistats2022}
\usepackage[colorlinks,
            linkcolor=red,
            anchorcolor=blue,
            citecolor=blue
            ]{hyperref}     
\usepackage{xcolor}     

\usepackage{smile}  
\allowdisplaybreaks
\newcommand{\la}{\langle}
\newcommand{\ra}{\rangle}
\newcommand{\qvalue}{Q}
\newcommand{\vvalue}{V}

\newcommand{\reward}{r}

\def \algname {\text{POWERS}}

\usepackage{enumitem}
\usepackage{wrapfig}
\usepackage{smile}

\begin{document}
\twocolumn[

\aistatstitle{Near-optimal Policy Optimization Algorithms for Learning Adversarial Linear Mixture MDPs}

\aistatsauthor{ Jiafan He  \And Dongruo Zhou \And  Quanquan Gu }

\aistatsaddress{ UCLA \And  UCLA \And  UCLA } ]

\begin{abstract}

  Learning Markov decision processes (MDPs) in the presence of the adversary is a challenging problem in reinforcement learning (RL). In this paper, we study RL in episodic MDPs with adversarial reward and full information feedback, where the unknown transition probability function is a linear function of a given feature mapping, and the reward function can change arbitrarily episode by episode. We propose an optimistic policy optimization algorithm POWERS and show that it can achieve $\tilde{O}(dH\sqrt{T})$ regret, where $H$ is the length of the episode, $T$ is the number of interactions with the MDP, and $d$ is the dimension of the feature mapping. Furthermore, we also prove a matching lower bound of $\tilde{\Omega}(dH\sqrt{T})$ up to logarithmic factors. Our key technical contributions are two-fold: (1) a new value function estimator based on importance weighting; and (2) a tighter confidence set for the transition kernel. They together lead to the nearly minimax optimal regret.
\end{abstract}

\section{INTRODUCTION}

The goal of reinforcement learning (RL) is to design a policy to maximize the reward through observation from interaction with the unknown environment. In reinforcement learning, the Markov decision process (MDP) \citep{puterman1994markov} is a typical model to describe the unknown environment and widely used to analyze the sequential dynamic environment. In this work, we consider episodic MDPs with a finite horizon. Traditional MDPs often assume the unknown transition probability function is fixed and the reward function is stochastic, which means the reward of each state-action pair follows an unknown stationary distribution. Yet, in many real world models, the reward function is not fixed and may change over time. 
In order to capture the changed or even adversarial reward, 
\citet{even2009online} first introduced the concept of adversarial MDP model and proposed MDP-Expert (MDP-E) algorithm, which attains $\tilde O{(\tau^2\sqrt{T})}$ regret with $\tau$ being the
mixing time of the MDP, for known transition probability function and full information of the reward function. In a concurrent work, \citet{yu2009markov} proposed an algorithm in the same setting and obtained $\tilde O{(T^{2/3})}$ regret. There is a line of follow up work studying RL for adversarial MDPs \citep{neu2010online,neu2012adversarial,zimin2013online,dekel2013better,rosenberg2019online,efroni2020optimistic}, which studies various settings depending on whether the transition probability function is known, and whether the feedback is full-information or bandit. Please see the related work section for a more detailed discussion. 

However, most existing works on adversarial MDP are in the tabular MDP setting, where both the number of actions and states are finite, and the action-value function is represented by a table. In many real-world RL problems, the state and action spaces are large or even infinite. A widely used method to overcome the curse of large state and action spaces is function approximation, which reparameterizes the tabular action-value function as a function over some feature mapping that maps the state and action to a low-dimensional space. Learning adversarial MDPs with linear function approximation is still understudied. 
Some existing works \citep{tamar2014scaling, zhang2020stability, wang2021online} study learning optimal policies for robust MDPs with function approximation under specific MDP assumptions, which are not directly applicable to general adversarial MDPs. A notable existing work is \citet{cai2020provably}, which studies general adversarial MDPs with linear function approximation. 
In particular, \citet{cai2020provably} proposed an optimistic variant of proximal policy optimization algorithm for the linear kernel MDP \citep{jia2020model,ayoub2020model,zhou2020provably} with unknown transition probability and full reward information in the adversarial setting, which achieves $\tilde{O}({\sqrt{d^2H^3T}})$ regret. Here $H$ is the length of the episode, $T$ is the number of interactions with the MDP and $d$ is the dimension of the feature mapping. 


In this paper, we seek a computationally efficient and statistically optimal algorithm for learning adversarial MDPs. The focus of this work is the unknown transition and full information setting. 
We first propose an algorithm called optimistic \textbf{P}olicy \textbf{O}ptimization \textbf{W}ith B\textbf{ER}nstein bonu\textbf{S} (\algname) for adversarial linear mixture MDP (See Definition for more details) with full information feedback. At a high level, our algorithm $\algname$ is similar to Optimistic-PPO (OPPO) algorithm \citep{cai2020provably}, which can also be seen as an extension of MDP-Expert (MDP-E) with linear function. More specifically, $\algname$ consists of two main steps in each round: (1) one-step least-square temporal difference (LSTD) learning along with exploration bonus for policy evaluation; and (2) mirror descent on the policy space for policy improvement. Our key algorithmic contributions include a weighted LSTD algorithm which takes into the variance of the Bellman residue into account, and a Bernstein-type bonus for exploration based on the principle of ``optimism-in-the-face-of-uncertainty'' \citep{abbasi2011improved}. We prove that $\algname$ achieves  
$\tilde{O}(dH\sqrt{T})$
regret with high probability, where $H$ is the length of the episode, $T$ is the number of interactions with the MDP and $d$ is the dimension of the feature mapping. We also prove an $\tilde{\Omega}(dH\sqrt{T})$ lower bound for adversarially learning linear kernel MDPs.
Our upper bound matches the lower bound up to logarithmic factors, which suggests that our algorithm is nearly minimax optimal. To the best of our knowledge, our algorithm is the first computationally efficient and statistical (nearly) optimal algorithm for adversarial MDPs in the unknown transition and full reward information setting.



\noindent\textbf{Notation.} We use lower case letters to denote scalars, and use lower and upper case boldface letters to denote vectors and matrices respectively.  For a vector $\xb\in \RR^d$ and matrix $\bSigma\in \RR^{d\times d}$, we denote by $\|\xb\|_2$ the Euclidean norm, $\|\xb\|_1 = \sum_{i=1}^d|x_i|$, and $\|\xb\|_{\bSigma}=\sqrt{\xb^\top\bSigma\xb}$. For two sequences $\{a_n\}$ and $\{b_n\}$, we write $a_n=O(b_n)$ if there exists an absolute constant $C$ such that $a_n\leq Cb_n$, and we write $a_n=\Omega(b_n)$ if there exists an absolute constant $C$ such that $a_n\geq Cb_n$. We use $\tilde O(\cdot)$ and $\tilde \Omega(\cdot)$ to further hide the logarithmic factors. For any $a \leq b \in \RR$, $x \in \RR$, let $[x]_{[a,b]}$ denote $a\cdot \ind(x \leq a) + x \cdot \ind (a \leq x \leq b) + b \cdot \ind (b \leq x)$, where $\ind(\cdot)$ is the indicator function. For a positive integer $n$, we use $[n]=\{1,2,..,n\}$ to denote the set of integers from $1$ to $n$. 

\section{RELATED WORK}
\label{sec:related}

\noindent\textbf{RL with adversarial reward.}
There is a long line of research on learning adversarial MDPs, where the reward function is adversarially chosen at the beginning of each episode and can change arbitrarily across different episodes \citep{even2009online,yu2009markov,gergely2010online,neu2010online,zimin2013online,neu2012adversarial,rosenberg2019online,rosenberg2019online1,wang2019optimism,cai2020provably,efroni2020optimistic}. The seminal works by \citet{even2009online,yu2009markov} are in the known transition probability and full reward information setting. In the known transition and bandit feedback on the reward setting, \citet{gergely2010online} proposed MDP-EXP3 algorithm and obtained $\tilde O{(T^{2/3})}$ regret. \citet{neu2010online} proposed Bandit O-SSP algorithm which achieves $\tilde O{(\sqrt{T}/\alpha)}$ regret with an addition assumption that all states are reachable with probability $\alpha>0$ for any policy. \citet{zimin2013online} further proposed O-REPS algorithm, which improves the regret from $\tilde O{(T^{2/3})}$ to $\tilde O{(\sqrt{T})}$ without any additional assumption.
In the unknown transition but full reward information setting, \citet{neu2012adversarial} proposed FPOP algorithm that achieves $\tilde O{(SA\sqrt{T})}$ regret. \citet{rosenberg2019online} proposed UC-O-REP algorithm and improved the regret to $\tilde O{(S\sqrt{AT})}$.
In the most challenging unknown transition and bandit reward feedback setting, \citet{rosenberg2019online1} proposed Shifted Bandit UC-O-REPS algorithm which achieves $\tilde O{(T^{3/4})}$ regret. \citet{rosenberg2019online1} also proposed Bounded Bandit UC-O-REPS algorithm and obtained $\tilde O{(\sqrt{T}/\alpha)}$ regret under the assumption that all states are reachable with probability $\alpha>0$ for any policy. \citet{jin2020learning} proposed UOB-REPS algorithm that achieves $\tilde O{(\sqrt{T})}$ regret without the additional assumption made by \citet{rosenberg2019online1}. 
The focus of this paper is the unknown transition but full reward information setting. 

\noindent\textbf{RL with linear function approximation.}
Recently, there emerges a large body of literature on solving MDP with linear function approximation. 
These works can be generally divided into three lines based on the specific assumption on the underlying MDP. The first line of work \citep{sun2019model,du2019provably1} is based on the low Bellman rank assumption \citep{jiang2017contextual}, which assumes a low-rank factorization of the Bellman error matrix. The second line of work \citep{wang2019optimism,he2020logarithmic,zanette2020frequentist} focuses on the linear MDP \citep{yang2019sample,jin2019provably}, where the transition probability function and reward function are parameterized as a linear function of a feature mapping $\bphi: \cS\times \cA \rightarrow \RR^d$. 
Later, \citet{zanette2020learning} made a weaker assumption called low inherent Bellman error and proposed Eleanor algorithm. 
The last line of work \citep{cai2020provably,yang2019reinforcement,he2020logarithmic,modi2019sample,zhou2020nearly} is based on the linear mixture/kernel MDP \citep{jia2020model,ayoub2020model,zhou2020provably,wu2021nearly}, where the transition probability function can be parameterized as a linear function of a feature mapping $\bphi: \cS  \times \cA  \times \cS\rightarrow \RR^d$. Note that none of the above work with linear function approximation can handle adversarially chosen reward with \citet{cai2020provably} being a notable exception. Our paper also considers the linear kernel MDP but with an adversarial reward function. 

\noindent\textbf{RL with policy gradient.}
Our work is also related to policy optimization and policy gradient methods \citep{williams1992simple,baxter2000direct,sutton1999policy,kakade2001natural,kakade2002approximately,kakade2003sample,bagnell2003covariant,schulman2015trust,schulman2017proximal,abbasi2019politex,abbasi2019exploration,cai2020provably,hao2020provably,efroni2020optimistic}, 
among which the most related works to ours are trust-region policy optimization \citep{schulman2015trust}, proximal policy optimization \citep{schulman2017proximal}, Politex \citep{abbasi2019politex}, EE-Politex \citep{abbasi2019exploration}, AAPI \citep{hao2020provably} and OPPO \citep{cai2020provably}. 
More specifically, \citet{cai2020provably} proposed the optimistic variant of the Proximal Policy Optimization algorithm for adversarial linear kernel MDP, which can be seen as an extension of MDP-E. \citet{abbasi2019politex} proposed Politex algorithm with least-squares policy evaluation for infinite-horizon average-reward MDPs, which can be seen a generalization of the MDP-E \citep{even2009online}. In fact, MDP-E is equivalent to TRPO/PPO \citep{schulman2015trust,schulman2017proximal} as shown by \citet{neu2017unified}. Our algorithm can be seen as a nontrivial extension of OPPO and MDP-E.


\section{PRELIMINARIES}\label{section 3}
\textbf{Time-inhomogeneous, episodic adversarial MDPs.}
In this paper, we consider a time-inhomogeneous, episodic Markov decision process (MDP), which is denoted by a tuple $M=M(\cS, \cA, H, \{\reward_h^k\}_{h\in[H],k\in[K]}, \{\PP_h\}_{h=1}^{H})$. Here $\cS$ is the state space, $\cA$ is the action space, $H$ is the length of the episode, $\reward_h^k: \cS \times \cA \rightarrow [0,1]$ is the deterministic reward function at stage $h$ of the $k$-th episode. 
$\PP_h(s'|s,a) $ is the transition probability function which denotes the probability for state $s$ to transfer to state $s'$ given action $a$ at stage $h$. For simplicity, we assume the reward function $\reward_h^k$ is adversarially chosen by the environment
at the beginning of the $k$-th episode and \emph{known after the episode $k$}. A policy $\pi=\{\pi_h\}_{h=1}^H$ is a collection of functions $\pi_h$, where each $\pi_h: \cS \rightarrow \Delta(\cA)$ is a function which maps a state $s$ to distributions over action set $\cA$ at stage $h$.
For each state-action pair $(s,a)\in \cS\times \cA$, we denote the action-value function $\qvalue_{k,h}^{\pi}$ and the value function $\vvalue_{k,h}^{\pi}$ as follows: 
\begin{align}
\qvalue_{k,h}^{\pi}(s,a) &= \reward_{h}^{k}(s,a)\notag\\
    & + \EE\bigg[\textstyle{\sum_{h' = h+1}^H} \reward_{h'}^{k}(s_{h'}, a_{h'})\bigg|s_{h} = s, a_{h} = a\bigg],\notag\\
\vvalue_{k,h}^{\pi}(s) &= \EE_{a\sim \pi_{h}(\cdot|s)}\big[\qvalue_{k,h}^{\pi}(s, a)\big],\vvalue_{k,H+1}^{\pi}(s)=0.\notag
\end{align}
In the definition of $\qvalue_{k,h}^{\pi}$, we denote by $\EE[\cdot]$ the expectation over the state-action sequences
 $(s_h,a_h,s_{h+1},a_{h+1},..,s_H,a_H)$, where $s_h=s,a_h=a$ and $s_{h'+1} \sim \PP_h(\cdot| s_{h'}, a_{h'}),\ a_{h'+1} \sim \pi_{h'+1}(\cdot|s_{h'+1})(h'=h,h+1,..,H-1)$. For simplicity, for any function $\vvalue: \cS \rightarrow \RR$, we denote   
 \begin{align}
 [\PP_h\vvalue](s,a)&=\EE_{s' \sim \PP_h(\cdot|s,a)}\vvalue(s'),\notag\\
  [\VV _h\vvalue](s,a)&=[\PP_h V^2](s,a)-\big([\PP_h V](s,a)\big)^2,\label{eq:variance}
\end{align}
where $\vvalue^2$ is a shorthand for the function whose value at state $s$ is $\big(\vvalue(s)\big)^2$.
 Using this notation, for policy $\pi$, we have the following Bellman equality $\qvalue_{k,h}^{\pi}(s,a) = \reward_h^k(s,a) + [\PP_h\vvalue_{k,h+1}^{\pi}](s,a)$.
 
In the \emph{online learning setting}, for each episode $k\ge 1$, at the beginning of the episode $k$, the agent determines a policy $\pi^k$ to be followed in this episode and we assume that the initial state $s_1^{k}$ is fixed across all episodes\footnote{While we study fixed initial state, our results readily can generalize to the case that the initial state $s_1^{k}$ is random chosen from a fixed distribution across all episodes. }. At each stage $h\in[H]$, the agent observe the state $s_h^k$, choose an action following the policy $a_h^k \sim \pi_h^k(\cdot|s_h^k)$ and observe the next state with $s_{h+1}^k \sim \PP_h(\cdot|s_h^k,a_h^k)$. For the adversarial environment, we focus on the expected regret, which is the expected loss of the
algorithm relative to the best-fixed policy in hindsight \citep{cesa2006prediction}:   
\begin{align}
    \text{Regret}(M,K)=\sup_{\pi}\textstyle{\sum_{k=1}^K}\big(\vvalue_{k,1}^{\pi}(s_1^k)-\vvalue_{k,1}^{\pi^k}(s_1^k)\big).\notag
\end{align}
For simplicity, we denote the optimal policy $\pi^*$ as $\pi^*=\sup_{\pi}\sum_{k=1}^K\vvalue_{k,1}^{\pi}(s_1^k)$. Therefore, we have the following Bellman optimally equation $ \qvalue_{k,h}^*(s,a) = \reward_h^k(s,a) + [\PP_h\vvalue_{k,h}^*](s,a)$,
where $\qvalue_{k,h}^*(s,a),\vvalue_{k,h}^*(s,a)$ are the corresponding optimal action-value function and value function. For any two policies $\pi$ and $\pi'$, we define the Kullback–Leibler divergence between them as follows $D_{KL}\big(\pi\|\pi'\big)=\sum_{a\in\cA}\pi(a)\log\big({\pi(a)}/{\pi'(a)}\big)$.

\noindent \textbf{Linear Mixture MDPs.} In this work, we focus on a special class of MDPs called \emph{linear mixture MDPs} \citep{jia2020model, ayoub2020model,zhou2020provably}, where the transition probability function is a linear function of a given feature mapping $\bphi: \cS  \times \cA  \times \cS \rightarrow \RR^d$. The formal definition of a linear kernel MDP is as follows:

\begin{definition}\label{assumption-linear}
$M(\cS, \cA, H, \{\reward_h^k\}_{h\in[H],k\in[K]}, \{\PP_h\}_{h=1}^{H})$ is called a inhomogenous, episode $B$-bounded linear mixture MDP if there exist a \emph{known} feature mapping $\bphi(s'|s,a): \cS \times \cA \times \cS \rightarrow \RR^d$ and an \emph{unknown} vector $\btheta_h \in \RR^d$ with $\|\btheta\|_2 \leq B$ \footnote{For any MDP $M$, parameter $B$ cannot be arbitrarily  small.  More specifically, parameter $B$ should satisfy $B \ge 1$.}, such that 
\begin{itemize} [leftmargin = *]
    \item For any state-action-next-state triplet $(s,a,s') \in \cS \times \cA \times \cS$, we have $\PP_h(s'|s,a) = \la \bphi(s'|s,a), \btheta_h\ra$;
    \item For any bounded function $\vvalue: \cS \rightarrow [0,1]$ and any tuple $(s,a)\in \cS \times \cA$, we have $\|\bphi_{{\vvalue}}(s,a)\|_2 \leq 1$, where $\bphi_{{\vvalue}}(s,a) = \sum_{s'\in\cS}\bphi(s'|s,a)\vvalue(s') \in \RR^d$. 
\end{itemize}
\end{definition}

Based on Definition \ref{assumption-linear}, we can see that for any linear mixture MDP $M$ and function $\vvalue: \cS \rightarrow \RR$, we have the following properties:   
\begin{align}
    [\PP_h\vvalue](s,a)
    &=\textstyle{\sum_{s'\in \cS}}\PP_h(s'|s,a)\vvalue(s')\notag\\
    &=\textstyle{\sum_{s'\in \cS}}\la \bphi(s'|s,a), \btheta_h\ra\vvalue(s')\notag\\
    &=\la\bphi_V(s,a), \btheta_h\ra,\label{linear-mixture-property}
\end{align}
and   
\begin{align}
    [\VV _h\vvalue](s,a)
    & = \textstyle{\sum_{s'\in \cS}}\PP_h(s'|s,a)\vvalue^2(s') \notag\\
    &\qquad - \big[\textstyle{\sum_{s'\in \cS}}\PP_h(\cdot|s,a)\vvalue(s')\big]^2\notag \\
    & = \la \bphi_{V^2}(s,a), \btheta_h\ra - [\la \bphi_{V}(s,a), \btheta_h\ra]^2.\label{linear-mixture-property_2}
\end{align}
\eqref{linear-mixture-property} and \eqref{linear-mixture-property_2} suggest that both the conditional expectation and the variance of a function $V$ can be calculated based on certain linear functions of different feature mappings, i.e., $\bphi_V$ and $\bphi_{V^2}$. Therefore, we can estimate them by estimating the corresponding linear functions. 

\section{THE PROPOSED ALGORITHM}\label{SECTION: 4}
In this section, we propose an algorithm $\algname$ to learn the episodic linear mixture MDP (see Definition \ref{assumption-linear}) with adversarial rewards, which is illustrated in Algorithm \ref{algorithm}. At a high level, $\algname$ is an improved version of the Optimistic-PPO (OPPO) algorithm \citep{cai2020provably} with a refined estimate of the action-value function $\qvalue_{k,h}(s,a)$. The $\algname$ can be divided into two phases: (1) policy improvement phase and (2) policy evaluation phase.

\begin{algorithm*}[t]
	\caption{$\algname$}\label{algorithm}
	\begin{algorithmic}[1]
	\REQUIRE Regularization parameter $\lambda$, learning rate $\alpha$.
	\STATE Set initial policy $\big\{\pi_h^0(\cdot|\cdot)\big\}_{h=1}^{H}$ as uniform distribution on the action set $\cA$
	\STATE For $h\in[H+1]$, set the initial value functions $\qvalue_{0,h}(\cdot,\cdot)\leftarrow 0,\vvalue_{0,h}(\cdot)\leftarrow 0$
	\STATE For $h\in[H]$, set $\hat{\bSigma}_{1,h},\tilde{\bSigma}_{1,h}\leftarrow \lambda \Ib, \hat{\bbb}_{1,h},\tilde{\bbb}_{1,h}\leftarrow \zero,\hat{\btheta}_{1,h},\tilde{\btheta}_{1,h}\leftarrow \zero$
	\FOR{$k=1,\ldots, K$}
	    \STATE Receive state $s_1^k$\label{algorithm:line5}
	    \FOR{$h=1,\ldots, H$}
	    \STATE Update the policy by $\pi_h^k(\cdot|\cdot) \propto \pi_h^{k-1}(\cdot|\cdot) \exp\big\{\alpha\qvalue_{k-1,h}(\cdot,\cdot)\big\}$\label{algorithm:line7}
	    \STATE Take action $a_h^k \sim \pi_h^k(\cdot|s_h^k)$ and receive next state $s_{h+1}^k\sim \PP_h(\cdot|s_h^k,a_h^k)$ \label{algorithm:line8}
	    \STATE Observe the adversarial reward function $\reward_h^k(\cdot,\cdot)$ 
	    \ENDFOR \label{algorithm:line10}
	    \STATE Set $\vvalue_{k,H+1}(\cdot)\leftarrow 0$ \label{algorithm:line11}
	    \FOR{$h=H,\ldots, 1$}
	    \STATE Set $\qvalue_{k,h}(\cdot,\cdot)\leftarrow  \Big[\reward_h^k(\cdot,\cdot)+\big\la\hat{\btheta}_{k,h},\bphi_{\vvalue_{k,h+1}}(\cdot,\cdot)\big\ra+\hat{\beta}_k\big\|\hat{\bSigma}_{k,h}^{-1/2}\bphi_{\vvalue_{k,h+1}}(\cdot,\cdot)\big\|_2\Big]_{[0,H-h+1]}$
	    \STATE Set $\vvalue_{k,h}(\cdot)\leftarrow \EE_{a\sim \pi_h^k(\cdot|\cdot)}[\qvalue_{k,h}(\cdot,a)]$
	    \STATE Set the estimated variance $ [\bar{\VV}_{k,h}\vvalue_{k,h+1}](s_h^k,a_h^k)$ as in \eqref{eq:estimated-variance} 
	    \STATE Set the bonus term $E_{k,h}$ as in \eqref{eq:bonus-term}
	    \STATE $\bar\sigma_{k,h}\leftarrow \sqrt{\max\big\{H^2/d,         [\bar{\VV}_{k,h}\vvalue_{k,h+1}](s_h^k,a_h^k)+E_{k,h}  \big\}}$\label{def:var}
	    \STATE $\hat{\bSigma}_{k+1,h}\leftarrow \hat{\bSigma}_{k,h}+\bar\sigma_{k,h}^{-2}\bphi_{\vvalue_{k,h+1}}(s_h^k,a_h^k)\bphi_{\vvalue_{k,h+1}}(s_h^k,a_h^k)^{\top}$\label{algorithm:line18}
	    \STATE $\hat{\bbb}_{k+1,h}\leftarrow \hat{\bbb}_{k,h}+\bar\sigma_{k,h}^{-2}\bphi_{\vvalue_{k,h+1}}(s_h^k,a_h^k)\vvalue_{k,h+1}(s_{h+1}^k)$
	    \STATE $\tilde{\bSigma}_{k+1,h}\leftarrow \tilde{\bSigma}_{k,h}+\bphi_{\vvalue^2_{k,h+1}}(s_h^k,a_h^k)\bphi_{\vvalue^2_{k,h+1}}(s_h^k,a_h^k)^{\top}$
	    \STATE $\tilde{\bbb}_{k+1,h}\leftarrow \tilde{\bbb}_{k,h}+\bphi_{\vvalue^2_{k,h+1}}(s_h^k,a_h^k)\vvalue^2_{k,h+1}(s_{h+1}^k)$
	    \STATE $\hat\btheta_{k+1,h}\leftarrow \hat{\bSigma}^{-1}_{k+1,h}\hat{\bbb}_{k+1,h}, \tilde\btheta_{k+1,h}\leftarrow \tilde{\bSigma}^{-1}_{k+1,h}\tilde{\bbb}_{k+1,h}$\label{algorithm:line22}
	    \ENDFOR \label{algorithm:line23}
	\ENDFOR
	\end{algorithmic}
\end{algorithm*}

\noindent\textbf{Policy improvement phase (Line \ref{algorithm:line5} to Line \ref{algorithm:line10}):} In the policy improvement phase, $\algname$ calculates its policy $\pi^k$ for the current episode, based on its previous policy $\pi^{k-1}$ using the proximal policy optimization (PPO) method \citep{schulman2017proximal}. In detail, let $s_1^k$ be the starting state at the $k$-th episode, then following PPO, we update $\pi^k$ as a solution to the following optimization problem:   
\begin{align}
    \pi^k\leftarrow \argmax_{\pi} [L_{k-1}(\pi) -\alpha^{-1}\tilde{D}_{KL}(\pi, \pi^{k-1})],\label{eq:xxx}
\end{align}
where    
\begin{align*}
    L_{k-1}(\pi) = \EE_{\pi^{k-1}}\bigg[\sum_{h=1}^H \la Q_{k-1, h}(s_h, \cdot), \pi^k_h(\cdot|s_h)\ra\bigg|s_1 = s_1^k\bigg]
\end{align*}
is proportional to the first-order Taylor approximation of $V_{k-1,h}^{\pi_h^{k-1}}$ at $\pi^{k-1}$, and replaces the action-value function $Q_{k-1, h}^{\pi_h^{k-1}}(\cdot,\cdot)$ by the estimated one $Q_{k-1,h}(\cdot,\cdot)$, and 
\begin{align*}
    &\tilde{D}_{KL}(\pi, \pi^{k-1})\notag\\
    &= \EE_{\pi^{k-1}}\bigg[\textstyle{\sum_{h=1}^H} D_{KL}(\pi_h(\cdot|s_h),\pi_h^k(\cdot|s_h))\bigg|s_1 = s_1^k\bigg]
\end{align*}
is the sum of KL-divergences between $\pi_h$ and $\pi_h^{k-1}$, which encourages $\pi^k$ to stay close to $\pi^{k-1}$ to ensure the above first-order Taylor approximation is accurate enough. The closed-form solution to \eqref{eq:xxx} is in Line \ref{algorithm:line7}. Here $\alpha>0$ is the step size of the exponential update. Note that the update rule in Line \ref{algorithm:line7} is also the same as the MDP-E algorithm \citep{even2009online}.
After obtaining $\pi^k$, $\algname$ chooses action $a_h^k$ based on the new policy $\pi_h^k$ and the current state $s_h^k$. It then observes the next state $s_{h+1}^k$ and the adversarial reward function $\reward_h^k(\cdot,\cdot)$. 

\noindent\textbf{Policy evaluation phase (Line \ref{algorithm:line11} to Line \ref{algorithm:line23}):} In the policy evaluation phase, $\algname$ evaluates the policy $\pi^k$ by constructing the action-value function $\qvalue_{k,h}$ and the value function $\vvalue_{k,h}$ for policy $\pi^k$ based on the observed data, which are optimistic estimates of the action-value function $\qvalue_{k,h}^{\pi^k}$ and the value function $\vvalue_{k,h}^{\pi^k}$ respectively. 


Specifically, for each episode $k\in[K]$ and each stage $h \in [H]$, $\algname$ maintains an estimator $\hat\btheta_{k,h}$ and an uncentered covariance matrix $\hat\bSigma_{k,h}$ based on the observed data before the $k$-th episode. Then $\algname$ recursively computes the optimistic $Q_{k,h}, V_{k,h}$ as follows:   
\begin{align}
    \qvalue_{k,h}(s,a)&=  \Big[\reward_h^k(s,a)+\big\la\hat{\btheta}_{k,h},\bphi_{\vvalue_{k,h+1}}(s,a)\big\ra \notag\\
    &\qquad+\hat{\beta}_k\big\|\hat{\bSigma}_{k,h}^{-1/2}\bphi_{\vvalue_{k,h+1}}(s,a)\big\|_2\Big]_{[0,H-h+1]},\notag\\
    \vvalue_{k,h}(s)&= \EE_{a\sim \pi_h^k(\cdot|s)}[\qvalue_{k,h}(s,a)],\notag
\end{align}
where $\hat{\beta_k}$ is the radius of the confidence ball defined as:   
    \begin{align}
  \hat{\beta_k}&= 8\sqrt{d\log(1+k/\lambda)\log(4k^2H/\delta)}\notag\\
  &\qquad +4\sqrt{d}\log(4k^2H/\delta)+\sqrt{\lambda}B.\notag
\end{align}
Now we illustrate how to construct the estimator $\hat\btheta_{k,h}$ and the covariance matrix $\hat\bSigma_{k,h}$, which is the key difference compared with OPPO proposed in \citet{cai2020provably}. Recall \eqref{linear-mixture-property}, we know that the expectation of the random variables $V_{k,h}(s_{k, h+1})$ can be written as a linear function  with weight vector $\btheta_h$. Therefore, a natural way to estimate $\btheta_h$ is to consider it as the unknown weight vector of a stochastic linear bandits problem with context $\bphi_{V_{k,h}}(s_h^k, a_h^k)$ and target $V_{k,h}(s_{k, h+1})$, and apply algorithms for linear bandits such as OFUL \citep{abbasi2011improved}, to obtain the estimator $\hat\btheta_{k,h}$. Such an approach is adopted by \citet{cai2020provably}. 

However, OFUL uses the vanilla linear regression to construct $\hat\btheta_{k,h}$, which is limited to the \emph{homoscedastic} noises case. For the linear mixture MDP, the noises are actually \emph{heteroscedastic} as each target enjoys different noises. Thus the vanilla linear regression is known as statistically inefficient \citep{kirschner2018information}. Inspired by \citet{kirschner2018information, zhou2020nearly}, we adopt the \emph{weighted linear regression} to construct $\hat\btheta_{k,h}$, which is the solution to the following weighted regression problem:    
\begin{align}
    \hat{\btheta}_{k,h}&\leftarrow \arg \min_{\btheta\in \RR^d}\lambda \|\btheta\|_2^2\notag\\& +\textstyle{\sum_{i=1}^{k-1}}\big[\la \bphi_{\vvalue_{i,h+1}}(s_h^i,a_h^i), \btheta \ra-\vvalue_{i,h+1}(s_{h+1}^i)\big]/\bar{\sigma}_{i,h}^2\notag,
\end{align}
where $\bar{\sigma}_{i,h}^2$ is the upper confidence bound of the variance $[\VV_h \vvalue_{i,h+1}(s_{h}^i,a_h^i)]$. $\hat\bSigma_{k,h}$ is the weighted ``covariance" matrix of $\bphi_{\vvalue_{i,h+1}}(s_h^i,a_h^i)$ weighted by $1/\bar{\sigma}_{i,h}^2$. The online update rules for $\hat{\btheta}_{k,h}$ and $\hat\bSigma_{k,h}$ are shown in Lines \ref{algorithm:line18} and \ref{algorithm:line22}. 

Next we show how to construct the variance upper bounds $\bar{\sigma}_{i,h}^2$. Due to \eqref{linear-mixture-property_2}, it suffices to estimate $\la\bphi_{\vvalue_{k,h+1}}(s_h^k,a_h^k),\btheta_{h}\ra$ and $\la\bphi_{\vvalue_{k,h+1}^2}(s_h^k,a_h^k),\btheta_{h}\ra$. For the first one, we use $\la\bphi_{\vvalue_{k,h+1}}(s_h^k,a_h^k),\hat\btheta_{k,h}\ra$ to estimate it. For the second one, we use $\la\bphi_{\vvalue_{k,h+1}^2}(s_h^k,a_h^k),\tilde\btheta_{k,h}\ra$, where $\tilde\btheta_{k,h}$ is the linear regression estimator with contexts $\bphi_{\vvalue_{i,h+1}^2}(s_h^i,a_h^i)$ and targets $\vvalue_{i,h+1}^2(s_h^{i+1})$. Its update rule is shown in Line~\ref{algorithm:line22}. Notice that the stochastic noise in the linear regression only comes from the stochastic transition probability $\PP_h$ rather than the adversarial reward. Since $\PP_h$ is fixed across different episodes, we can bound the estimation error and specify the choice of $\bar\sigma_{k,h}^2$ in the following lemma.

\begin{lemma}\label{LEMMA: CONCENTRATE} We define the 
 the estimated variance $[\bar{\VV}_{k,h}\vvalue_{k,h+1}](s_h^k,a_h^k)$  as 
\begin{align}
    [\bar{\VV}_{k,h}\vvalue_{k,h+1}](s_h^k,a_h^k)&=\Big[\big\la\bphi_{\vvalue^2_{k,h+1}}(s_h^k,a_h^k),\tilde{\btheta}_{k,h}\big\ra\Big]_{[0,H^2]}\notag\\&-\Big[\big\la\bphi_{\vvalue_{k,h+1}}(s_h^k,a_h^k),\hat{\btheta}_{k,h}\big\ra_{[0,H]}\Big]^2,\label{eq:estimated-variance}
\end{align}
then with probability at least $1-3\delta$, for all $k\in[K], h\in[H]$, we have   
\begin{align}
   \big\|[\bar\VV_h\vvalue_{k,h+1}](s_h^k,a_h^k)         -[\VV_h\vvalue_{k,h+1}](s_h^k,a_h^k)\big\|\leq E_{k,h},\notag
\end{align}
where $E_{k,h}$ is defined as   
\begin{align}
E_{k,h}&=\min \Big\{\tilde{\beta}_k\big\|\tilde{\bSigma}_{k,h}^{-1/2}\bphi_{\vvalue^2_{k,h+1}}(s_h^k,a_h^k)\big\|_2,H^2\Big\}\notag\\&\qquad+\min \Big\{2H\bar{\beta}_k\big\|\hat{\bSigma}_{k,h}^{-1/2}\bphi_{\vvalue_{k,h+1}}(s_h^k,a_h^k)\big\|_2,H^2\Big\},\notag\\
    \tilde{\beta}_k&=8H^2\sqrt{d\log\big(1+kH^4/(d\lambda)\big)\log(4k^2H/\delta)} \notag\\&\qquad +4H^2\log(4k^2H/\delta)+\sqrt{\lambda}B,\notag\\
    \bar{\beta}_k&=8d\sqrt{\log(1+k/\lambda)\log(4k^2H/\delta)} \notag\\&\qquad +4\sqrt{d}\log(4k^2H/\delta)+\sqrt{\lambda}B.\label{eq:bonus-term}
\end{align}
For the estimator $\hat{\btheta}_{k,h}$, we have 
\begin{align}
    \btheta_h \in \cC_{k,h}= \{\btheta:\big\|\hat{\bSigma}_{k,h}^{1/2}(\btheta - \hat{\btheta}_{k,h})\big\|\leq \hat{\beta}_k\}.\label{eq:variance-ucb}
\end{align}
\end{lemma}
By Lemma \ref{LEMMA: CONCENTRATE}, we know that in order to guarantee $\bar\sigma_{k,h}^2$ is an upper bound of the variance, it suffices to set it as $[\bar{\VV}_{k,h}\vvalue_{k,h+1}](s_h^k,a_h^k) + E_{k,h}$. Finally, due to the technical reason, we set $\bar\sigma_{k,h}^2$ as   
\begin{align}
    \bar\sigma_{k,h}= \sqrt{\max\big\{H^2/d,         [\bar{\VV}_{k,h}\vvalue_{k,h+1}](s_h^k,a_h^k)+E_{k,h}  \big\}}\notag.
\end{align}
Furthermore, according to \eqref{eq:variance-ucb}, we have   
\begin{align} 
    \qvalue_{k,h}(s,a)&=\Big[\reward_h^k(s,a)+\max_{\btheta \in\cC_{k,h}}\big\la \bphi_{\vvalue_{k,h+1}},\btheta\big\ra\Big]_{[0,H-h+1]}\notag\\
    &\ge\Big[\reward_h^k(s,a)+\big\la \bphi_{\vvalue_{k,h+1}},\btheta_h\big\ra\Big]_{[0,H-h+1]},\notag
\end{align}
and by \eqref{linear-mixture-property}, it is easy to show that the optimistic action-value function $\qvalue_{k,h}(s,a)$ and the optimistic value function $\vvalue_{k,h}(s)$ are indeed upper bounds of the true action-value function $\qvalue_{k,h}^{\pi^k}$ and the true value function $\vvalue_{k,h}^{\pi^k}$, respectively.

\subsection{Computational complexity}
The computational complexity of $\algname$ is related to the property of the given feature mapping $\bphi(s'|s,a)$ and we consider a special class of linear mixture MDPs studied by \citet{yang2019reinforcement,zhou2020provably,zhou2020nearly}. For this special class of linear mixture MDPs, we have 
\begin{align}
    [\bphi(s'|s,a)]_i = [\bpsi(s')]_i\cdot [\bmu(s,a)]_i,\forall i \in[d],\notag
\end{align}
where $\bpsi:\cS \rightarrow \RR^d$ and $\bmu:\cS\times \cA \rightarrow \RR^d$. Under this setting, for each function $\vvalue$, the vector $\bphi_{\vvalue}(s,a)$ can be written as the product of $\bmu(s,a)$ and $\sum_{s'\in \cS} \bpsi(s')\vvalue(s')$. Furthermore, the term $\sum_{s'\in \cS} \bpsi(s')\vvalue(s')$ can be estimated by Monte Carlo method and in this work, we assume an access to the oracle $\cO$ which can compute the term $\sum_{s'\in \cS} \bpsi(s')\vvalue(s')$.
We also assume the size of action space is finite ($|\cA|<\infty$) and analyze the computational complexity of $\algname$ in the sequel.

 Recall that  $\algname$ can be divided into two phases: (1) policy improvement; and (2) policy evaluation. For the policy evaluation phase, in order to compute the vector $\bphi_{\vvalue_{k,h+1}}(s_h^k,a_h^k)$ and the vector $\bphi_{\vvalue^2_{k,h+1}}(s_h^k,a_h^k)$, $\algname$ needs to compute the term $\sum_{s'\in \cS} \bpsi(s')\bphi_{\vvalue_{k,h+1}}(s')$ and $\sum_{s'\in \cS} \bpsi(s')\vvalue^2_{k,h+1}(s')$, which need two accesses to the oracle $\cO$. Given the vector $\bphi_{\vvalue_{k,h+1}}(s_h^k,a_h^k)$  and $\bphi_{\vvalue_{k,h+1}}(s_h^k,a_h^k)$, the covariance matrix can be computed in $O(d^2)$ time, and the estimators $\hat{\btheta}_{k+1,h}$ and $\tilde{\btheta}_{k+1,h}$ can be computed in $O(d^3)$ time. Therefore, the policy evaluation phase can be computed in $O(d^3HK)$ time with $O(HK)$ accesses to the oracle $\cO$.

  In the policy improvement phase, $\algname$ will update the policy $\pi_h^k(s|a)$ for each state-action pair $(s,a)$:
\begin{align} 
  \pi_h^{k+1}(a|s) \propto \pi_h^{k}(a|s) \exp\big\{\alpha\qvalue_{k,h}(s,a)\big\},\notag
\end{align}
which leads to an $O(|\cS||\cA|K)$ computation complexity. In order to make the computation more efficient for large state space or even continuous state space $\cS$, an alternative approach is to calculate the policy $\pi_h^k(\cdot|s_h^k)$ directly. More specifically, in Line \ref{algorithm:line8},
  we only need the value of policy $\pi_h^k$ for state $s_h^k$ , which can be calculated as follows
\begin{align}
    \pi_h^{k}(a|s_h^k) &\propto  \exp\bigg\{\alpha\sum_{i=1}^{k-1}\qvalue_{i,h}(s_h^k,a)\bigg\}\notag\\
    &\propto \exp\bigg\{\alpha\sum_{i=1}^{k-1}\Big[\reward_h^i(s_h^k,a)+\big\la\hat{\btheta}_{i,h},\bphi_{\vvalue_{i,h+1}}(s_h^k,a)\big\ra \notag\\&\qquad+\hat{\beta}_i\big\|\hat{\bSigma}_{i,h}^{-1/2}\bphi_{\vvalue_{i,h+1}}(s_h^k,a)\big\|_2\Big]_{[0,H-h+1]}\bigg\}\notag.
\end{align}
Therefore, given the covariance matrix $\hat{\bSigma}_{i,h}$, the estimator $\hat{\btheta}_{i,h}$ and the term $\sum_{s'\in \cS} \bpsi(s')\vvalue_{i,h+1}(s')$, the policy $\pi_h^k(\cdot|s_h^k)$ can be computed in $O\big(d^3K|\cA|\big)$ time complexity and it will take in total $O\big(d^3HK^2|\cA|\big)$ time complexity to compute all policies $\pi_h^k(\cdot|s_h^k)$ for $k\in [K]$ and $h\in [H]$.
By taking the best of both computing approaches, the time complexity for the policy improvement phase of $\algname$ is $O\big(\min\big(d^3HK^2|\cA|,|\cS||\cA|K\big)\big)$.

 
Combining the time complexity of the two phases, the total time complexity of POWERS is $O\big(\min\big(d^3HK^2|\cA|,|\cS||\cA|K\big) + d^3HK\big)$ with $O(HK)$ accesses to the oracle $\cO$.


\section{MAIN RESULTS}
In this section, we provide the regret bound for our algorithm $\algname$.  Here, $T=KH$ is the number of interactions with the MDP.

\begin{theorem}\label{THM:1}
For any linear mixture MDP $M$, if we set the parameter $\lambda=1/B^2$ in $\algname$, then with probability at least $1-6\delta$, the regret of $\algname$ is upper bounded as follows:
\begin{align}
    \text{Regret}(M,K) &\leq\tilde{O}\big(\alpha TH^2+\alpha^{-1} H\log |\cA|+ \sqrt{d^2H^2T} \notag\\&\qquad +\sqrt{dH^3T}+d^2H^3+d^{2.5}H^{2.5}\big).\notag
\end{align}
\end{theorem}
\begin{remark}
If we set the learning rate $\alpha=O\big(\sqrt{\log |\cA|/{(H^2K)}}\big)$ in $\algname$,
then Theorem~\ref{THM:1} suggests that with high probability, the regret of $\algname$ is upper bounded by $\tilde{O}(\sqrt{d^2H^2T}+\sqrt{dH^3T}+d^2H^3+d^{2.5}H^{2.5}+\sqrt{H^3 \log |\cA| T})$. When $T\ge d^3H^3, d\ge H,d\ge \log |\cA|$, the regret can be simplified as $\tilde{O}(dH\sqrt{T})$. Compared with the result of \citet{cai2020provably}, the $\algname$ improves the upper bound of regret by a factor of $\sqrt{H}$. In addition, compared with the $\text{UCRL-VTR+}$ algorithm proposed by \citet{zhou2020nearly} for episodic linear mixture MDPs with deterministic rewards, our algorithm $\algname$ provides a robustness guarantee against adversarial rewards while achieving the same regret guarantee $\tilde{O}(dH\sqrt{T})$.
\end{remark}

The following theorem gives a lower bound of the regret for any algorithms for learning the adversarial linear mixture MDPs.
\begin{theorem}\label{THM:2}
Suppose $B\ge 2, d\ge 4, H\ge 3,K\ge (d-1)^2H/2$, then for any algorithm,  there exists a time-inhomogenous, episodic $B$-bounded adversarial linear mixture MDP $M$, such that the expected regret for this MDP is lower bounded by $\Omega(dH\sqrt{T})$.
\end{theorem}
\begin{remark}
Theorem \ref{THM:2} suggests that when the number of episodes $K$ is large enough, for any algorithm, the regret of learning time-inhomogenous episodic adversarial linear mixture MDPs is at least $\Omega(dH\sqrt{T})$. Furthermore, the lower bound of regret in Theorem \ref{THM:2} matches the upper bound in Theorem \ref{THM:1} up to logarithmic factors, which suggests that $\algname$ is nearly minimax optimal for learning adversarial linear mixture MDPs.
\end{remark}

\section{PROOF OVERVIEW}\label{SECTION: 6}

In this section, we provide the proof sketch of Theorem~\ref{THM:1}. The complete proof is deferred to the supplementary material.
Our proof is based on Lemma \ref{LEMMA: CONCENTRATE}, which says with high probability, the true parameter $\btheta_h$ is contained in the confidence set $\cC_{k,h}$. For simplicity, we denote by $\cE$ the event when the result of Lemma \ref{LEMMA: CONCENTRATE} holds and due to Lemma \ref{LEMMA: CONCENTRATE}, we have $\Pr(\cE)\ge 1-3\delta$. 
Recall the definition of regret, we have
\begin{align}
      \text{Regret}(K)
      &=\underbrace{\textstyle{\sum_{k=1}^K}\big(\vvalue_{k,1}^{*}(s_1^k)-\vvalue_{k,1}(s_1^k)\big)}_{I_1}\notag\\&\qquad+\underbrace{\textstyle{\sum_{k=1}^K}\big(\vvalue_{k,1}(s_1^k)-\vvalue_{k,1}^{\pi^k}(s_1^k)\big)}_{I_2}\notag.
\end{align}
Therefore, the proof can be divided into two main steps.

\noindent\textbf{Step 1: Bounding the term $I_1$.} 
\begin{lemma}\label{LEMMA:TELESCOPE-SUM}
On the event $\cE$, for all $k\in[K]$, we have
\begin{align}
\vvalue_{k,1}^{*}(s_1^k)-\vvalue_{k,1}(s_1^k)&\leq \EE\bigg[\sum_{h=1}^H\Big\{\EE_{a\sim \pi^*_{h}(\cdot|s_h)}\big[\qvalue_{k,h}(s_h, a)\big]\notag\\&-\EE_{a\sim \pi^k_{h}(\cdot|s_h)}\big[\qvalue_{k,h}(s_h, a)\big]\big|s_1=s_1^k\Big\}\bigg].\notag
\end{align}
Here $\EE[\cdot|s=s_1^k]$ is the expectation with respect to the randomness of the state-action
sequence $\{(s_h,a_h)\}_{h=1}^H$, where $a_{h} \sim \pi^*_{h}(\cdot|s_{h})$ and $s_{h+1} \sim \PP_h(\cdot| s_{h}, a_{h})$.
\end{lemma}
Lemma \ref{LEMMA:TELESCOPE-SUM} suggests that the regret can be decomposed as the sum of the advantages at different stages $h$. Furthermore, since the initial state $s_1^k$ and the optimal policy $\pi^*$ is fixed across different episode $k$, the state-action sequence $(s_1,a_1,..,s_H,a_H)$ induced by the policy $\pi^*$ follows the same distribution across different episode $k$.

In each episode $k$, the $\algname$ update the policy $\pi_h^k$ by the following rule:
\begin{align} 
   \pi_h^{k+1}(a|s) \propto \pi_h^{k}(a|s) \exp\big\{\alpha\qvalue_{k,h}(s,a)\big\}.\notag
\end{align}
By the above update rule, we have the following lemma.
\begin{lemma}\label{LEMMA:ONE-STEP-DESCENT}
On the event $\cE$, for all $k\in[K]$, $h\in[H]$, $s\in\cS$, we have
\begin{align}
    &\EE_{a\sim \pi^*_{h}(\cdot|s)}\big[\qvalue_{k,h}(s, a)\big]-\EE_{a\sim \pi^k_{h}(\cdot|s)}\big[\qvalue_{k,h}(s, a)\big]\notag\\
    &\leq \frac{\alpha H^2}{2}+\alpha^{-1}\Big(D_{KL}\big(\pi_h^*(\cdot|s)\|\pi_h^k(\cdot|s)\big)\notag\\
    &\qquad -D_{KL}\big(\pi_h^*(\cdot|s)\|\pi_h^{k+1}(\cdot|s)\big)\Big).\notag
\end{align}
\end{lemma}

Substituting the result of Lemma \ref{LEMMA:ONE-STEP-DESCENT} into the result of Lemma \ref{LEMMA: CONCENTRATE}, we have the following upper bound for the term $I_1$:
\begin{align}
    I_1&\leq\frac{\alpha KH^3}{2}+\sum_{k=1}^K\alpha^{-1}\EE\bigg[\sum_{h=1}^H\Big\{D_{KL}\big(\pi_h^*(\cdot|s_h)\|\pi_h^k(\cdot|s_h)\big)\notag\\
    &\qquad -{KL}\big(\pi_h^*(\cdot|s_h)\|\pi_h^{k+1}(\cdot|s_h)\big)\Big\}\bigg]\notag\\
    &=\frac{\alpha KH^3}{2}+\alpha^{-1}\EE\bigg[\sum_{h=1}^H \Big\{D_{KL}\big(\pi_h^*(\cdot|s_h)\|\pi_h^1(\cdot|s_h)\big)\notag\\
    &\qquad-D_{KL}\big(\pi_h^*(\cdot|s_h)\|\pi_h^{K+1}(\cdot|s_h)\big)\Big\}\bigg]\notag\\
    &\leq\frac{\alpha KH^3}{2}+\alpha^{-1}\EE\bigg[\sum_{h=1}^HD_{KL}\big(\pi_h^*(\cdot|s_h)\|\pi_h^1(\cdot|s_h)\big)\bigg]\notag\\
    &\leq \frac{\alpha KH^3}{2}+\alpha^{-1} H\log |\cA|,\notag
\end{align}
where $s_1$ is the fixed initial state, $a_{h} \sim \pi^*_{h}(\cdot|s_{h}),s_{h+1} \sim \PP_h(\cdot| s_{h}, a_{h})$, the second inequality holds due to Kullback–Leibler divergence is non-negative and the third inequality holds due to the fact that initial policy $\pi_h^1$ is uniform over the action space $\cA$.

\noindent\textbf{Step 2: Bounding the term $I_2$.}


\begin{lemma}\label{LEMMA:TRANSITION}
On the event $\cE$, for all $k\in[K],h\in[H]$, we have
\begin{align}
    &\qvalue_{k,h}(s_h^k,a_h^k)-\qvalue^{\pi^k}_{k,h}(s_h^k,a_h^k)\notag\\
    &\leq \big[\PP_h(\vvalue_{k,h+1}-\vvalue_{k,h+1}^{\pi^k})\big](s_h^k,a_h^k)\notag\\
    &\qquad +2\hat{\beta}_k \bar\sigma_{k,h}\min \Big\{\big\|\hat{\bSigma}_{k,h}^{-1/2}\bphi_{\vvalue_{k,h+1}}(s_h^k,a_h^k)/\bar\sigma_{k,h}\big\|_2,1\Big\}\notag.
\end{align}
\end{lemma}
Lemma \ref{LEMMA:TRANSITION} suggests that the difference between the true action-value function $\qvalue^{\pi^k}_{k,h}(s_h^k,a_h^k)$ and the estimated value function $\qvalue_{k,h}(s_h^k,a_h^k)$ can be bounded by the expected difference at the next-stage difference. Furthermore, for the expected difference at the next-stage and the exact difference at the next-stage, we have the following equation
\begin{align}
    &\big[\PP_h(\vvalue_{k,h+1}-\vvalue_{k,h+1}^{\pi^k})\big](s_h^k,a_h^k)\notag\\
    &=\qvalue_{k,h+1}(s_{h+1}^k,a_{h+1}^k)-\qvalue_{k,h+1}^{\pi^k}(s_{h+1}^k,a_{h+1}^k)\notag\\
    &\qquad +A_{h,k}+B_{h+1,k} ,\notag
\end{align}
where $A_{h,k}=[\PP_h(\vvalue_{k,h+1}-\vvalue_{k,h+1}^{\pi^k})\big](s_h^k,a_h^k)-\big(\vvalue_{k,h+1}(s_{h+1}^{k})-\vvalue_{k,h+1}^{\pi^k}(s_{h+1}^{k})$ is the noise from the state transition and $B_{h,k}=\EE_{a\sim \pi_h^k(\cdot|s_h^k)}\big[\qvalue_{k,h}(s_h^k,a)-\qvalue_{k,h}^{\pi^k}(s_h^k,a)\big]-\big(\qvalue_{k,h}(s_h^k,a_h^k)-\qvalue_{k,h}^{\pi^k}(s_h^k,a_h^k\big)$ is the noise from the stochastic policy. These noises form a martingale difference sequence and we define two high probability events for them:
\begin{align}
    \cE_1&=\bigg\{\forall h\in[H], \sum_{k=1}^K\sum_{h'=h}^{H}A_{h',k}+B_{h',k}\notag\\
    &\qquad \leq 4H\sqrt{T\log(H/\delta)} \bigg\},\notag\\
    \cE_2&=\bigg\{\sum_{k=1}^K\sum_{h=1}^{H}A_{h,k}\leq 2H\sqrt{2T\log(1/\delta)}\bigg\}.\notag
\end{align}
Then according to the Azuma–Hoeffding inequality, we have $\Pr(\cE_1)\ge 1-\delta$ and $\Pr(\cE_2)\ge 1-\delta$. Furthermore, on the events $\cE_1$ and $\cE_2$, we can telescope the inequality in Lemma \ref{LEMMA:TRANSITION} over the $K$ episodes, and obtain the following lemma. 
\begin{lemma}\label{LEMMA:TRANSITION2}
On the event $\cE\cap\cE_1\cap \cE_2$, for all $h\in[H]$, we have
\begin{align}
    &\sum_{k=1}^K\big(\vvalue_{k,h}(s_h^k)-\vvalue_{k,h}^{\pi^k}(s_h^k)\big)\notag\\
    &\leq 2\hat{\beta}_K\sqrt{\sum_{k=1}^K\sum_{h=1}^H\bar\sigma_{k,h}^2}\sqrt{2Hd\log(1+K/\lambda)}\notag\\
    &\quad+4H\sqrt{T\log (H/\delta)},\notag
\end{align}
and
\begin{align}
    &\sum_{k=1}^K\sum_{h=1}^H \big[\PP_h(\vvalue_{k,h+1}-\vvalue_{k,h+1}^{\pi^k})\big](s_h^k,a_h^k)\notag\\
    &\leq 2H\hat{\beta}_K\sqrt{\sum_{k=1}^K\sum_{h=1}^H\bar\sigma_{k,h}^2}\sqrt{2Hd\log(1+K/\lambda)}\notag\\
    &\quad+4H^2\sqrt{T\log (H/\delta)}.\notag
\end{align}
\end{lemma}
Lemma \ref{LEMMA:TRANSITION2} shows that the regret can be upper bounded by the total estimated variance $\sum_{k=1}^K\sum_{h=1}^H\bar\sigma_{k,h}^2$. 
For the total variance $\sum_{k=1}^K\sum_{h=1}^H [\VV_h \vvalue_{k,h+1}^{\pi^k}](s_h^k,a_h^k)$, we introduce the high probability event $\cE_3$:
\begin{small}
\begin{align}
    \cE_3=\Big\{\sum_{k=1}^K\sum_{h=1}^H [\VV_h \vvalue_{k,h+1}^{\pi^k}](s_h^k,a_h^k)\leq 3HT+3H^3\log(1/\delta)\Big\}.\notag
\end{align}
\end{small}
Since the stochastic noise in the linear regression only comes from the stochastic transition probability $\PP_h$ rather than the adversarial reward, and $\PP_h$ is fixed across different episodes, Lemma C.5 in \citet{jin2018q} suggests that $\Pr(\cE_3)\ge 1-\delta$ and on the event $\cE\cap\cE_1\cap\cE_2\cap\cE_3$, the following lemma gives a upper bound of the total estimated variance.
\begin{lemma}\label{LEMMA: TOTAL-ESTIMATE-VARIANCE}
On the event $\cE\cap\cE_1\cap\cE_2\cap\cE_3$, we have
\begin{align}
    &\sum_{k=1}^K\sum_{h=1}^H\bar\sigma_{k,h}^{2}\notag \leq 2HT/d+179HT\notag\\
    &\quad +165d^3H^4 \log^2(4K^2H/\delta)\log^2(1+KH^4/\lambda)\notag\\
    &\quad +2062H^5d^2 \log^2(4K^2H/\delta)\log^2(1+K/\lambda)\notag.
\end{align}
\end{lemma}
Combining the results of Lemmas \ref{LEMMA:TRANSITION2} and \ref{LEMMA: TOTAL-ESTIMATE-VARIANCE}, the term $I_2$ can be upper bounded by
\begin{align}
    I_2= \tilde{O}(\sqrt{dH^3T}+\sqrt{d^2H^2T}+d^2H^3+d^{2.5}H^{2.5})\notag.
\end{align}
Finally, combining the upper bounds of $I_1$ and $I_2$, we have
\begin{align*}
    \text{Regret}(K)&=I_1+I_2\notag\\
    &=\tilde{O}(\sqrt{dH^3T}+\sqrt{d^2H^2T}+d^2H^3+d^{2.5}H^{2.5}\notag\\
    &\qquad +\alpha H^2+\alpha^{-1}H \log|\cA|).
\end{align*}

\section{EXPERIMENTS}

In this section, we carry out experiments to evaluate the empirical performance of our algorithm $\algname$.

In this experiment, we construct the MDP $M$ with dimension $d=5$ and episode length $H=20$. In this MDP, the state space $\cS$ consists of $H+2$ different states $s_1,..,s_{H+2}$ and the action space $\cA=\{-1,1\}^{d-1}$ consists of $2^{d-1}$ different actions. For each stage $h\in [H]$ and episode $k\in[K]$, the adversarial reward function $\reward_h^k$ satisfies that $\reward_h^k(s_h,\ba)=0,\ (1\leq h\leq H+1)$ and $\reward_h^k(s_{H+2},\ba)=1$. For each stage $h\in[H]$ and corresponding transition probability function $\PP_h$, $s_{H+1}$ and $s_{H+2}$ are absorbing states. For other states $s_h (1\leq h\leq H)$, the transition probability satisfies that
\begin{align}
    &\PP_h(s_{h+1}|s_h,\ab)=0.95-\la 0.01\cdot\textbf{1}_{d-1} ,\ab\ra,\notag\\
    &\PP_h(s_{H+2}|s_h,\ab)=0.05+\la 0.01\cdot\textbf{1}_{d-1} ,\ab\ra\notag,
\end{align}
where each $\textbf{1}_{d-1}$ is a $(d-1)$-dimensional vector of all ones. In the experiment, we set the regularization parameter $\lambda=1$, learning rate $\alpha=1$ and use grid search to select the parameters $\hat \beta$, $\tilde \beta$, $\bar \beta$ which obtain the best performance. We compare our algorithm with two baseline benchmarks : \textbf{Random} (uniformly and randomly choose actions from $\cA$) and \textbf{OPPO} \citep{cai2020provably}. The regrets of different algorithms for the first $1000$ episodes 
averaged over $20$ runs are plotted in Figure \ref{fig:1}.
\begin{figure}[htbp]
  \begin{center}
    \includegraphics[width=0.45\textwidth]{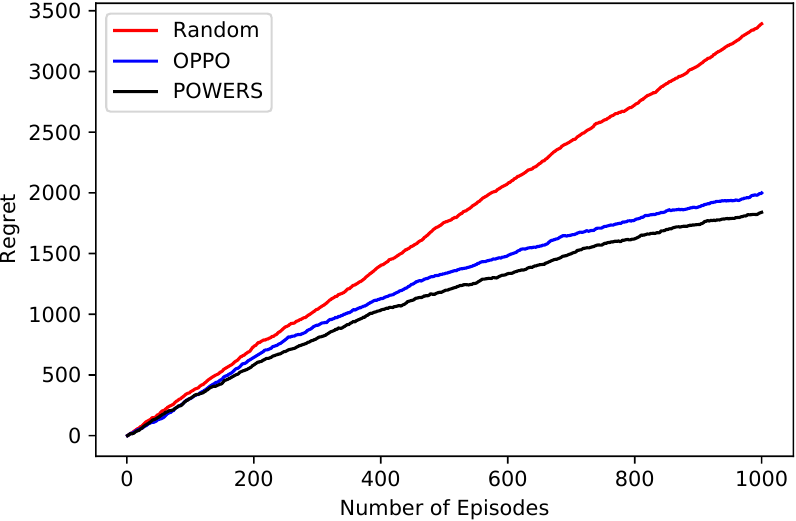}
  \end{center}
   \caption{Cumulative regret comparison in the first $1000$ episodes of different algorithms. Results are averaged over $20$ runs.}
  \label{fig:1}
\end{figure}

From Figure \ref{fig:1}, it can be seen that both OPPO \citep{cai2020provably} and $\algname$ obtain a sub-linear cumulative regret. Moreover, we can see that $\algname$ outperforms the OPPO algorithm. This is consistent with our  theoretical result in Theorem \ref{THM:1}, which suggests that $\algname$ has a better regret bound than OPPO thanks to the use of Bernstein-type bonus.

\section{CONCLUSION AND FUTURE WORK}\label{sec:conclusion}
In this work, we considered learning adversarial Markov decision processes under the linear mixture MDP assumption. We proposed a novel algorithm $\algname$ and proved that with high probability, the regret of $\algname$ is upper bounded by $\tilde{O}(dH\sqrt{T})$, which matches the lower bound up to logarithmic factors. Currently, our work requires the full information feedback of the reward and it remains an open problem that if there exists a provably efficient algorithm for learning adversarial linear mixture MDPs with bandit-feedback on the reward. We leave it as future work.



\section*{Acknowledgments}
We thank the anonymous reviewers for their helpful comments. 
Part of this work was done when JH, DZ and QG participated the Theory of Reinforcement Learning program at the Simons Institute for the Theory of Computing in Fall 2020. JH, DZ and QG are partially supported by the National Science Foundation CAREER Award 1906169, IIS-1904183, and AWS Machine Learning Research Award. The views and conclusions contained in this paper are those of the authors and should not be interpreted as representing any funding agencies.

\bibliographystyle{ims}
\bibliography{reference}


\clearpage
\appendix

\thispagestyle{empty}

\onecolumn \makesupplementtitle

\section{PROOF OF THE MAIN RESULTS}\label{section: proof of main}
\subsection{Proof of Theorem \ref{THM:1}}
In this section, we provide the proof of Theorems \ref{THM:1} and first propose the following lemmas.
\begin{lemma}[Azuma–Hoeffding inequality, \citealt{cesa2006prediction}]\label{lemma:azuma}
Let $\{x_i\}_{i=1}^n$ be a martingale difference sequence with respect to a filtration $\{\cG_{i}\}$ satisfying $|x_i| \leq M$ for some constant $M$, $x_i$ is $\cG_{i+1}$-measurable, $\EE[x_i|\cG_i] = 0$. Then for any $0<\delta<1$, with probability at least $1-\delta$, we have 
\begin{align}
    \sum_{i=1}^n x_i\leq M\sqrt{2n \log (1/\delta)}.\notag
\end{align} 
\end{lemma}

\begin{lemma}\label{LEMMA:TRANSITION1}
On the event $\cE$, for all $k\in[K],h\in[H],s\in \cS,a\in \cA$, we have
\begin{align}
    \qvalue_{k,h}(s,a)&\ge 
    \reward_h^k(s,a)+\big[\PP_h\vvalue_{k,h+1}\big](s,a).\notag
\end{align}
Furthermore, on the event $\cE$, for all $k\in[K],h\in[H],s\in \cS,a\in \cA$, we have
\begin{align}
    \qvalue_{k,h}(s,a)-\qvalue^{\pi^k}_{k,h}(s,a)&\ge 0,\vvalue_{k,h}(s)-\vvalue^{\pi^k}_{k,h}(s)\ge 0.\notag
\end{align}
\end{lemma}
\begin{proof}[Proof of Lemma \ref{LEMMA:TRANSITION1}]
For the lower bound of $\qvalue_{k,h}(s,a)-\qvalue^{\pi^k}_{k,h}(s,a)$, we have
\begin{align}
    &\reward_h^k(s,a)+\big\la\hat{\btheta}_{k,h},\bphi_{\vvalue_{k,h+1}}(s,a)\big\ra+\hat{\beta}_k\big\|\hat{\bSigma}_{k,h}^{-1/2}\bphi_{\vvalue_{k,h+1}}(s,a)\big\|_2\notag\\
    &=\reward_h^k(s,a)+\big[\PP_h\vvalue_{k,h+1}\big](s,a)+\big\la\hat{\btheta}_{k,h}-\btheta_h,\bphi_{\vvalue_{k,h+1}}(s,a)\big\ra+\hat{\beta}_k\big\|\hat{\bSigma}_{k,h}^{-1/2}\bphi_{\vvalue_{k,h+1}}(s,a)\big\|_2\notag\\
    &\ge \reward_h^k(s,a)+\big[\PP_h\vvalue_{k,h+1}\big](s,a)+ \hat{\beta}_k\big\|\hat{\bSigma}_{k,h}^{-1/2}\bphi_{\vvalue_{k,h+1}}(s,a)\big\|_2-\big\|\hat{\bSigma}_{k,h}^{1/2}(\btheta-\hat{\btheta}_{k,h})\big\|_2\big\|\hat{\bSigma}_{k,h}^{-1/2}\bphi_{\vvalue_{k,h+1}}(s,a)\big\|_2\notag\\
    &\ge \reward_h^k(s,a)+\big[\PP_h\vvalue_{k,h+1}\big](s,a),\label{eq:51}
\end{align}
where the first inequality holds due to Cauchy-Schwarz inequality and the second inequality holds due to the definition of event $\cE$.
Therefore, we have
\begin{align}
    \qvalue_{k,h}(s,a)&= \Big[\reward_h^k(s,a)+\big\la\hat{\btheta}_{k,h},\bphi_{\vvalue_{k,h+1}}(s,a)\big\ra+\hat{\beta}_k\big\|\hat{\bSigma}_{k,h}^{-1/2}\bphi_{\vvalue_{k,h+1}}(s,a)\big\|_2\Big]_{[0,H-h+1]}\notag\\
   &\ge \min\Big\{\reward_h^k(s,a)+\big[\PP_h\vvalue_{k,h+1}\big](s,a),H-h+1\Big\}\notag\\
    &\ge \reward_h^k(s,a)+\big[\PP_h\vvalue_{k,h+1}\big](s,a),\label{eq:52}
\end{align}
where the first inequality holds due to \eqref{eq:51} and the second inequality holds due to $\reward_h^k(s,a)+\big[\PP_h\vvalue_{k,h+1}\big](s,a)\leq 1+ (H-h)=H-h+1$.
Now, we prove the second part of Lemma \ref{LEMMA:TRANSITION1} by induction. The statement holds for stage $h=H+1$, since
\begin{align}
    \vvalue_{k,h}(s)=\vvalue^{\pi^k}_{k,h}(s)=0.\notag
\end{align}
When the second part of Lemma \ref{LEMMA:TRANSITION1} holds for stage $h+1$, we have
\begin{align}
   \qvalue_{k,h}(s,a)&\ge  \reward_h^k(s,a)+\big[\PP_h\vvalue_{k,h+1}\big](s,a)\notag\\
   &\ge \reward_h^k(s,a)+\big[\PP_h\vvalue^{\pi^k}_{k,h+1}\big](s,a)\\
   &=\qvalue^{\pi^k}(s,a),\notag
\end{align}
where the first inequality holds due to \eqref{eq:52} and the second inequality holds due to the induction assumption. Furthermore, we have
\begin{align}
    \vvalue_{k,h}(s)=\EE_{a\sim \pi_h^k(\cdot|s)}[\qvalue_{k,h}(s,a)]\ge \EE_{a\sim \pi_h^k(\cdot|s)}[\qvalue^{\pi^k}_{k,h}(s,a)]=\vvalue^{\pi^k}_{k,h}(s).\notag
\end{align}
Therefore, we finish the induction step and complete the proof of Lemma \ref{LEMMA:TRANSITION1}.
\end{proof}

\begin{lemma}\label{lemma: sum-UCB}[Lemma 11, \citealt{abbasi2011improved}]
Let $\{\xb_t\}_{t=1}^{\infty}$ be a sequence in $\RR^d$, $\Vb_0=\lambda \Ib$ and define $\Vb_t=\Vb_0+\sum_{i=1}^t \xb_i\xb_i^{\top}$. If $\|\xb_i\|_2\leq L$ holds for each $i$, then for each $t$, we have 
\begin{align}
    \sum_{i=1}^t\min \big\{1,\|\xb_i\|_{\Vb_{i-1}^{-1}}\big\}\leq 2d\log\bigg(\frac{d\lambda+tL^2}{d\lambda}\bigg) \notag.
\end{align}
\end{lemma}

\begin{proof} [Proof of Theorem \ref{THM:1}]
For the regret, we have
\begin{align}
      \text{Regret}(K)&=\sup_{\pi}\sum_{k=1}^K\big(\vvalue_{k,1}^{\pi}(s_1^k)-\vvalue_{k,1}^{\pi^k}(s_1^k)\big)\notag\\
      &=\sum_{k=1}^K\big(\vvalue_{k,1}^{*}(s_1^k)-\vvalue_{k,1}^{\pi^k}(s_1^k)\big)\notag\\
      &=\underbrace{\sum_{k=1}^K\big(\vvalue_{k,1}^{*}(s_1^k)-\vvalue_{k,1}(s_1^k)\big)}_{I_1}+\underbrace{\sum_{k=1}^K\big(\vvalue_{k,1}(s_1^k)-\vvalue_{k,1}^{\pi^k}(s_1^k)\big)}_{I_2}\notag.
\end{align}

For the term $I_1$, applying Lemma \ref{LEMMA:TELESCOPE-SUM}, we have
\begin{align}
    I_1&=\sum_{k=1}^K\big(\vvalue_{k,1}^{*}(s_1^k)-\vvalue_{k,1}(s_1^k)\big)\notag\\
    &\leq \sum_{k=1}^K\EE\bigg[\sum_{h=1}^H\Big\{\EE_{a\sim \pi^*_{h}(\cdot|s_h)}\big[\qvalue_{k,h}(s_h, a)\big]-\EE_{a\sim \pi^k_{h}(\cdot|s_h)}\big[\qvalue_{k,h}(s_h, a)\big]\Big\}\big|s_1=s_1^k\bigg]\notag\\
    &\leq \sum_{k=1}^K\EE\bigg[\sum_{h=1}^H\Big\{\frac{\alpha H^2}{2}+\alpha^{-1}\Big(D_{KL}\big(\pi_h^*(\cdot|s_h)\|\pi_h^k(\cdot|s_h)\big)-D_{KL}\big(\pi_h^*(\cdot|s_h)\|\pi_h^{k+1}(\cdot|s_h)\big)\Big)\Big\}\Big|s_1=s_1^k\bigg]\notag\\
    &=\frac{\alpha KH^3}{2}+\sum_{k=1}^K\alpha^{-1}\EE\bigg[\sum_{h=1}^H\Big\{D_{KL}\big(\pi_h^*(\cdot|s_h)\|\pi_h^k(\cdot|s_h)\big)-D_{KL}\big(\pi_h^*(\cdot|s_h)\|\pi_h^{k+1}(\cdot|s_h)\big)\Big\}\big|s_1=s_1^k\bigg],\label{eq:11}
\end{align}
where $a_{h} \sim \pi^*_{h}(\cdot|s_{h}),s_{h+1} \sim \PP_h(\cdot| s_{h}, a_{h})$, the first inequality holds due to Lemma \ref{LEMMA:TELESCOPE-SUM} and the second inequality holds due to Lemma \ref{LEMMA:ONE-STEP-DESCENT}. For Kullback–Leibler divergence $D_{KL}\big(\pi_h^*(\cdot|s_h)\|\pi_h^1(\cdot|s_h)\big)$, we have $0\leq D_{KL}\big(\pi_h^*(\cdot|s_h)\|\pi_h^1(\cdot|s_h)\big)$ and
\begin{align}
    D_{KL}\big(\pi_h^*(\cdot|s_h)\|\pi_h^1(\cdot|s_h)\big)&=\sum_{a\in\cA}\pi_h^*(a|s_h)\log\bigg(\frac{\pi_h^*(a|s_h)}{\pi_h^1(a|s_h)}\bigg)\notag\\
    &=\sum_{a\in\cA}\pi_h^*(a|s_h)\log \big(\pi_h^*(a|s_h)\times |\cA|\big)\notag\\
    &=\log|\cA|+\sum_{a\in\cA}\pi_h^*(a|s_h)\log \big(\pi_h^*(a|s_h)\big)\notag\\
    &\leq \log|\cA|,\label{eq:12}
\end{align}
where the first equation holds due to $\pi_h^1(a|s_h)=1/|\cA|$ and the  inequality holds on due to $0\leq \pi_h^*(a|s_h)\leq 1$.
Substituting \eqref{eq:12} into \eqref{eq:11}, we have
\begin{align}
    I_1&\leq\frac{\alpha KH^3}{2}+\sum_{k=1}^K\alpha^{-1}\EE\bigg[\sum_{h=1}^H\Big\{D_{KL}\big(\pi_h^*(\cdot|s_h)\|\pi_h^k(\cdot|s_h)\big)-D_{KL}\big(\pi_h^*(\cdot|s_h)\|\pi_h^{k+1}(\cdot|s_h)\big)\Big\}\bigg]\notag\\
    &=\frac{\alpha KH^3}{2}+\alpha^{-1}\EE\bigg[\sum_{h=1}^H\Big\{D_{KL}\big(\pi_h^*(\cdot|s_h)\|\pi_h^1(\cdot|s_h)\big)-D_{KL}\big(\pi_h^*(\cdot|s_h)\|\pi_h^{K+1}(\cdot|s_h)\big)\Big\}\bigg]\notag\\
    &\leq\frac{\alpha KH^3}{2}+\alpha^{-1}\EE\bigg[\sum_{h=1}^H\Big\{D_{KL}\big(\pi_h^*(\cdot|s_h)\|\pi_h^1(\cdot|s_h)\big)\Big\}\bigg]\notag\\
    &\leq \frac{\alpha KH^3}{2}+\alpha^{-1} H\log |\cA|,\label{eq:13}
\end{align}
where $s_1$ is the fixed initial state, $a_{h} \sim \pi^*_{h}(\cdot|s_{h}),s_{h+1} \sim \PP_h(\cdot| s_{h}, a_{h})$, the first inequality holds due to \eqref{eq:11} and the second inequality holds due to Kullback–Leibler divergence is non-negative and the third inequality holds due to \eqref{eq:12}.
For the term $I_2$, we have
\begin{align}
 I_2&=\sum_{k=1}^K\big(\vvalue_{k,1}(s_1^k)-\vvalue_{k,1}^{\pi^k}(s_1^k)\big)\notag\\
 &\leq 2\hat{\beta}_K\sqrt{\sum_{k=1}^K\sum_{h=1}^H\bar\sigma_{k,h}^2}\sqrt{2Hd\log(1+K/\lambda)}+4H\sqrt{T\log (H/\delta)}\notag\\
 &\leq 56\sqrt{dH^3T}{\log(4K^2H/\delta)}\log(1+K/\lambda) +492\sqrt{d^2H^2T}{\log(4K^2H/\delta)}\log(1+K/\lambda)\notag\\
 &\qquad+1670 d^2H^3{\log^2(4K^2H/\delta)}\log^2(1+K/\lambda)+473d^{2.5}H^{2.5}{\log^2(4K^2H/\delta)}\log^2(1+KH^4/\lambda),\label{eq:14}
\end{align}
where the first inequality holds due to Lemma \ref{LEMMA:TRANSITION2} and the second inequality holds due to Lemma \ref{LEMMA: TOTAL-ESTIMATE-VARIANCE} with the fact that $\sqrt{a+b+c+d}\leq \sqrt{a}+\sqrt{b}+\sqrt{c}+\sqrt{d}$. Substituting \eqref{eq:13} and \eqref{eq:14} into \eqref{eq:12}, we finish the proof of Theorem \ref{THM:1}.
\end{proof}

\subsection{Proof of Theorem \ref{THM:2}}
In this section, we provide the proof of the lower bounds of the regret and the lower bound is based on previous work \citep{zhou2020provably,zhou2020nearly}.
\begin{proof}[Proof of Theorem \ref{THM:2}]
To prove the lower bound, we construct a series of hard-to-learn adversarial MDPs introduced by \citet{zhou2020provably,zhou2020nearly}. To be more specific, the state space $\cS$ consists of states $s_1,..,s_{H+2}$, where $s_{H+1}$ and $s_{H+2}$ are absorbing states. The action space $\cA=\{-1,1\}^{d-1}$ consists of $2^{d-1}$ different actions. The adversarial reward function $\reward_h^k$ satisfies that $\reward_h^k(s_h,\ba)=0(1\leq h\leq H+1)$ and $\reward_h^k(s_{H+2},\ba)=1$. For the transition probability function $\PP_h$, $s_{H+1}$ and $s_{H+2}$ are absorbing states, which will always stay at the same state, and for other state $s_h(1\leq h\leq H)$, we have
\begin{align}
    &\PP_h(s_{h+1}|s_h,\ab)=1-\delta-\la\bmu_h,\ab\ra,\notag\\
    &\PP_h(s_{H+2}|s_h,\ab)=\delta+\la\bmu_h,\ab\ra\notag,
\end{align}
where each $\bmu_h\in \{-\Delta,\Delta\}^{d}$ with $\Delta=\sqrt{\delta/K}/(4\sqrt{2})$ and $\delta=1/H$. Furthermore, these hard-to-learn adversarial MDPs can be represented as linear mixture MDPs with the following feature mapping $\bphi: \cS \times \cS  \times \cA \rightarrow \RR^d$ and vector $\btheta_h$:
\begin{align}
    &\bphi(s_{h+1}|s_h,\ab) = \big(\alpha(1-\delta),-\beta\ab\big), h\in[H],\notag\\
    &\bphi(s_{H+2}|s_h,\ab) = \big(\alpha\delta,\beta\ab\big), h\in[H],\notag\\
    &\bphi(s_{h+1}|s_h,\ab) = \big(\alpha,\zero\big), h\in[H],\notag\\
    &\bphi(s_{h+1}|s_h,\ab) = (0,\zero), h\in[H],\notag\\
    &\btheta_h=(1/\alpha,\bmu_h/\beta),h\in[H]\notag,
\end{align}
where $\zero=0^{d-1}$ is a $(d-1)$-dimensional vector of all zeros, $\alpha=\sqrt{1/\big(1+(d-1)\Delta\big)}$ and $\beta=\sqrt{\Delta/\big(1+(d-1)\Delta\big)}$. Since $B\ge 2$ and $K\ge (d-1)^2H/2$, we have $\|\btheta_h\|_2=\big(1+\Delta(d-1)\big)^2\leq B$ and these hard-to-learn MDPs are $B$-bounded linear mixture MDPs.

Since the adversarial reward function is fixed across different episode $k$, the value function $\vvalue^{\pi}_{k,h}$ is also fixed for each policy $\pi$ and the optimal policy $\pi^*$ is pick the action $\ba^*=\bmu_{h'}/\Delta$ at state $s_{h'}(1\leq h'\leq H)$. Therefore, the adversarial MDP will degenerate to non-adversarial MDP and the adversarial regret is the same as the non-adversarial regret.
For the lower bound of the non-adversarial regret, Theorem 5.6 \citep{zhou2020nearly} shows that for any algorithm, if $B\ge 2, H\ge 3,d\ge 4$ and $K\ge (d-1)^2H/2$, then there exist a parameter $\bmu^*=\{\bmu_1^*,..,\bmu_H^*\}$ 
such that the expected regret is lower bounded by
\begin{align}
    \text{Regret(K)}= \Omega({dH\sqrt{T}}).\notag
\end{align}
Therefore, we finish the proof of Theorem \ref{THM:2}.
\end{proof}

\section{PROOF OF LEMMAS IN SECTIONS \ref{SECTION: 4} AND \ref{SECTION: 6}}
\subsection{Proof of Lemma \ref{LEMMA: CONCENTRATE}}
We need the following Lemmas: 
\begin{lemma}[Bernstein inequality for vector-valued martingales, \citealt{zhou2020nearly}]
Let $\{\mathcal{F}_t\}_{t=1}^{\infty}$ be a filtration and $(\xb_t,\eta_t)_{t\ge 1}$ be a  stochastic process so that  $\xb_t\in \RR^d$ is $\mathcal{F}_t$-measurable and $\eta_{t}\in \RR$ is  $\mathcal{F}_{t+1}$-measurable. For constant $R,L,\sigma,\lambda>0,\bmu^*\in \RR^d$, let $y_t=\la\xb_t,\bmu^*\ra+\eta_t$ and suppose that
\begin{align}
    |\eta_t|\leq R,\EE[\eta_t|\mathcal{F}_t]=0,\EE[\eta_t|\mathcal{F}_t]\leq \sigma^2,\|\xb_t\|_2\leq L.\notag
\end{align}
Then, for any $0\leq \delta\leq 1$, with probability at least $1-\delta$, we have
\begin{align}
    \forall t>0, \bigg\|\sum_{i=1}^t\xb_i\eta_i\bigg\|_{\bSigma_{t}^{-1}}\leq \beta_t,\|\bmu_t-\bmu^*\|_{\bSigma_{t}}\leq \beta_t+\sqrt{\lambda}\|\bmu^*\|_2,\notag
\end{align}
where $\bSigma_t=\lambda \Ib+\sum_{i=1}^t \xb\xb^{\top},\bbb_t=\sum_{i=1}^t\xb_iy_i,\bmu_t=\bSigma_t^{-1}\bbb_t$ and
\begin{align}
    \beta_t=8\sigma\sqrt{d\log\big(1+tL^2/(d\lambda)\big)\log(4t^2/\delta)}+4R\log(4t^2/\delta).\notag
\end{align}
\end{lemma}

\begin{proof}[Proof of Lemma \ref{LEMMA: CONCENTRATE}]
For each $h\in[H]$, by the definition of $[\bar\VV_h\vvalue_{k,h+1}](s_h^k,a_h^k)$ in \eqref{eq:estimated-variance} and $[\VV_h\vvalue_{k,h+1}](s_h^k,a_h^k)$ in \eqref{eq:variance}, we have
\begin{align}
    &[\bar\VV_h\vvalue_{k,h+1}](s_h^k,a_h^k)-[\VV_h\vvalue_{k,h+1}](s_h^k,a_h^k)\notag\\
    &=\Big[\big\la\bphi_{\vvalue^2_{k,h+1}}(s_h^k,a_h^k),\tilde{\btheta}_{k,h}\big\ra,H^2\Big]_{[0,H^2]}-\Big[\big\la\bphi_{\vvalue_{k,h+1}}(s_h^k,a_h^k),\hat{\btheta}_{k,h}\big\ra_{[0,H]}\Big]^2\notag\\
    &\qquad -\Big\{[\PP_h \vvalue^2_{k,h+1}](s_h^k,a_h^k)-\big([\PP_h \vvalue_{k,h+1}](s_h^k,a_h^k)\big)^2 \Big\}\notag\\
    &=\underbrace{\Big[\big\la\bphi_{\vvalue^2_{k,h+1}}(s_h^k,a_h^k),\tilde{\btheta}_{k,h}\big\ra\Big]_{[0,H^2]}-[\PP_h \vvalue^2_{k,h+1}](s_h^k,a_h^k)}_{I_1}\notag\\
    &\qquad+\underbrace{\big([\PP_h \vvalue_{k,h+1}](s_h^k,a_h^k)\big)^2-\Big[\big\la\bphi_{\vvalue_{k,h+1}}(s_h^k,a_h^k),\hat{\btheta}_{k,h}\big\ra_{[0,H]}\Big]^2}_{I_2}.\label{eq:sum-bonus}
\end{align}
For the term $I_1$, we have
\begin{align}
    |I_1|&=\bigg|\Big[\big\la\bphi_{\vvalue^2_{k,h+1}}(s_h^k,a_h^k),\tilde{\btheta}_{k,h}\big\ra\Big]_{[0,H^2]}-[\PP_h \vvalue^2_{k,h+1}](s_h^k,a_h^k)\bigg|\notag\\
    &\leq \Big|\big\la\bphi_{\vvalue^2_{k,h+1}}(s_h^k,a_h^k),\tilde{\btheta}_{k,h}\big\ra-[\PP_h \vvalue^2_{k,h+1}](s_h^k,a_h^k)\Big|\notag\\
    &=\Big|\big\la\bphi_{\vvalue^2_{k,h+1}}(s_h^k,a_h^k),\tilde{\btheta}_{k,h}\big\ra-\big\la\bphi_{\vvalue^2_{k,h+1}}(s_h^k,a_h^k),{\btheta_h}\big\ra\Big|\notag\\
    &= \Big|\big\la\bphi_{\vvalue^2_{k,h+1}}(s_h^k,a_h^k),\tilde{\btheta}_{k,h}-{\btheta_h}\big\ra\Big|\notag\\
    &\leq \big\|\bphi_{\vvalue^2_{k,h+1}}(s_h^k,a_h^k)\big\|_{\tilde{\bSigma}_{k,h}^{-1}}\big\|\tilde{\btheta}_{k,h}-{\btheta_h}\big\|_{\tilde{\bSigma}_{k,h}},\label{eq:I_1}
\end{align}
where the first inequality holds due to $0\leq [\PP_h \vvalue^2_{k,h+1}](s_h^k,a_h^k)\leq H^2 $ and the second inequality holds due to Cauchy-Schwarz inequality. For the term $\|\tilde{\btheta}_{k,h}-{\btheta_h}\big\|_{\tilde{\bSigma}_{k,h}}$, we apply Lemma \ref{LEMMA: CONCENTRATE} with $\xb_t=\bphi_{\vvalue^2_{t,h+1}}(s_h^t,a_h^t), \eta_t=\vvalue^2_{t,h+1}(s_{h+1}^t)-[\PP_h\vvalue^2_{t,h+1}](s_{h}^t,a_{h}^t).$ For $\xb_t,\eta_t$, we have the following property
\begin{align}
    \|\xb_t\|_2&=\big\|\bphi_{\vvalue^2_{t,h+1}}(s_h^t,a_h^t)\big\|_2\leq \max_{s'}\vvalue_{t,h+1}^2(s')\leq{H^2},\notag\\
    \EE[\eta_t|\mathcal{F}_t]&=0,|\eta_t|=\big|\vvalue^2_{t,h+1}(s_{h+1}^t)-[\PP_h\vvalue^2_{t,h+1}](s_{h}^t,a_{h}^t)\big|\leq H^2,\notag\\
    \EE[\eta^2_t|\mathcal{F}_t]&\leq H^4.\notag
\end{align}
Therefore, with probability at least $1-\delta/H$, for all $k\in[K]$, we have
\begin{align}
    \big\|\tilde{\btheta}_{k,h}-{\btheta_h}\big\|_{\tilde{\bSigma}_{k,h}}\leq 8H^2\sqrt{d\log\big(1+kH^4/(d\lambda)\big)\log(4k^2H/\delta)}+4H^2\log(4k^2H/\delta)+\sqrt{\lambda}B.\label{eq:square-beta}
\end{align}
Substituting \eqref{eq:square-beta} into \eqref{eq:I_1}, we have
\begin{align}
    |I_1|&\leq  \big\|\bphi_{\vvalue^2_{k,h+1}}(s_h^k,a_h^k)\big\|_{\tilde{\bSigma}_{k,h}^{-1}}\Big(8H^2\sqrt{d\log\big(1+kH^4/(d\lambda)\big)\log(4k^2H/\delta)}+4H^2\log(4k^2H/\delta)+\sqrt{\lambda}B\Big)\notag\\
    &=\tilde{\beta}_k\big\|\tilde{\bSigma}_{k,h}^{-1/2}\bphi_{\vvalue^2_{k,h+1}}(s_h^k,a_h^k)\big\|_2\notag.
\end{align}
Since both two terms of $I_1$ belong to the interval $[0,H^2],$ we have 
\begin{align}
    |I_1|&\leq \min\Big\{\tilde{\beta}_k\big\|\tilde{\bSigma}_{k,h}^{-1/2}\bphi_{\vvalue^2_{k,h+1}}(s_h^k,a_h^k)\big\|_2,H^2\Big\}.\label{eq:I_1-end}
\end{align}
For the term $I_2$, we have
\begin{align}
    |I_2|&=\bigg|\big([\PP_h \vvalue_{k,h+1}](s_h^k,a_h^k)\big)^2-\Big[\big\la\bphi_{\vvalue_{k,h+1}}(s_h^k,a_h^k),\hat{\btheta}_{k,h}\big\ra_{[0,H]}\Big]^2\bigg|\notag\\
    &=\bigg|\Big[\big\la\bphi_{\vvalue_{k,h+1}}(s_h^k,a_h^k),\hat{\btheta}_{k,h}\big\ra\Big]_{[0,H]}-[\PP_h \vvalue_{k,h+1}](s_h^k,a_h^k)\bigg|\notag\\
    &\qquad\times \bigg|\Big[\big\la\bphi_{\vvalue_{k,h+1}}(s_h^k,a_h^k),\hat{\btheta}_{k,h}\big\ra\Big]_{[0,H]}+[\PP_h \vvalue_{k,h+1}](s_h^k,a_h^k)\bigg|\notag\\
    &\leq 2H\bigg|\Big[\big\la\bphi_{\vvalue_{k,h+1}}(s_h^k,a_h^k),\hat{\btheta}_{k,h}\big\ra\Big]_{[0,H]}-[\PP_h \vvalue_{k,h+1}](s_h^k,a_h^k)\bigg|\notag\\
    &\leq2H\Big|\big\la\bphi_{\vvalue_{k,h+1}}(s_h^k,a_h^k),\hat{\btheta}_{k,h}\big\ra-\big\la\bphi_{\vvalue_{k,h+1}}(s_h^k,a_h^k),{\btheta_h}\big\ra\Big|\notag\\
    &=2H \Big|\big\la\bphi_{\vvalue_{k,h+1}}(s_h^k,a_h^k),\hat{\btheta}_{k,h}-{\btheta_h}\big\ra\Big|\notag\\
    &\leq 2H\big\|\bphi_{\vvalue_{k,h+1}}(s_h^k,a_h^k)\big\|_{\hat{\bSigma}_{k,h}^{-1}}\big\|\hat{\btheta}_{k,h}-{\btheta_h}\big\|_{\hat{\bSigma}_{k,h}},\label{eq:I_2}
\end{align}
where the first inequality and second inequality holds due to $0\leq [\PP_h \vvalue_{k,h+1}](s_h^k,a_h^k)\leq H $ and the third inequality holds due to Cauchy-Schwarz inequality. For the term $\|\hat{\btheta}_{k,h}-{\btheta_h}\big\|_{\hat{\bSigma}_{k,h}}$, we apply Lemma \ref{LEMMA: CONCENTRATE} with $\xb_t=\bar\sigma_{k,h}^{-1}\bphi_{\vvalue_{t,h+1}}(s_h^t,a_h^t), \eta_t=\bar\sigma_{k,h}^{-1}\vvalue_{t,h+1}(s_{h+1}^t)-\bar\sigma_{k,h}^{-1}[\PP_h\vvalue_{t,h+1}](s_{h}^t,a_{h}^t).$ For $\xb_t,\eta_t$, we have following property
\begin{align}
    \|\xb_t\|_2&=\big\|\bar\sigma_{k,h}^{-1}\bphi_{\vvalue_{t,h+1}}(s_h^t,a_h^t)\big\|_2\leq \bar\sigma_{k,h}^{-1}\max_{s'}|\vvalue_{t,h+1}(s')|\leq \sqrt{d},\notag\\
    \EE[\eta_t|\mathcal{F}_t]&=0,|\eta_t|=\big|\bar\sigma_{k,h}^{-1}\vvalue_{t,h+1}(s_{h+1}^t)-\bar\sigma_{k,h}^{-1}[\PP_h\vvalue_{t,h+1}](s_{h}^t,a_{h}^t)\big|\leq \sqrt{d},\notag\\
    \EE[\eta^2_t|\mathcal{F}_t]&\leq \sup \eta_t^2\leq d.\notag
\end{align}
Therefore, with probability at least $1-\delta/H$, for all $k\in[K]$, we have
\begin{align}
    \big\|\hat{\btheta}_{k,h}-{\btheta_h}\big\|_{\hat{\bSigma}_{k,h}}\leq 8d\sqrt{\log\big(1+kH^4/(d\lambda)\big)\log(4k^2H/\delta)}+4\sqrt{d}\log(4k^2H/\delta)+\sqrt{\lambda}B.\label{eq:square-beta1}
\end{align}
Substituting \eqref{eq:square-beta1} into \eqref{eq:I_2}, we have
\begin{align}
    |I_2|&\leq  2H\big\|\bphi_{\vvalue^2_{k,h+1}}(s_h^k,a_h^k)\big\|_{\hat{\bSigma}_{k,h}^{-1}}\Big(8d\sqrt{\log\big(1+kH^4/(d\lambda)\big)\log(4k^2H/\delta)}+4\sqrt{d}\log(4k^2H/\delta)+\sqrt{\lambda}B\Big)\notag\\
    &=\hat{\beta}_k\big\|\hat{\bSigma}_{k,h}^{-1/2}\bphi_{\vvalue^2_{k,h+1}}(s_h^k,a_h^k)\big\|_2\notag.
\end{align}
Since both two terms of $I_2$ belong to the interval $[0,H^2],$ we have 
\begin{align}
    |I_2|&\leq \min\Big\{2H\bar{\beta}_k\big\|\hat{\bSigma}_{k,h}^{-1/2}\bphi_{\vvalue^2_{k,h+1}}(s_h^k,a_h^k)\big\|_2,H^2\Big\}.\label{eq:I_2-end}
\end{align}
Substituting \eqref{eq:I_1-end} and \eqref{eq:I_2-end} into \eqref{eq:sum-bonus}, with probability at least $1-2\delta/H$, we have
\begin{align}
    \big|[\bar\VV_h\vvalue_{k,h+1}](s_h^k,a_h^k)-[\VV_h\vvalue_{k,h+1}](s_h^k,a_h^k)\big|=|I_1+I_2|\leq |I_1|+|I_2|\leq E_{k,h},\label{eq:estimate-variance-UCB}
\end{align}
where 
\begin{align}
    E_{k,h}=\min \Big\{\tilde{\beta}_k\big\|\tilde{\bSigma}_{k,h}^{-1/2}\bphi_{\vvalue^2_{k,h+1}}(s_h^k,a_h^k)\big\|_2,H^2\Big\}+\min \Big\{2H\bar{\beta}_k\big\|\hat{\bSigma}_{k,h}^{-1/2}\bphi_{\vvalue_{k,h+1}}(s_h^k,a_h^k)\big\|_2,H^2\Big\}\notag.
\end{align}
We apply Lemma \ref{LEMMA: CONCENTRATE} again with $\xb_t=\bar\sigma_{k,h}^{-1}\bphi_{\vvalue_{t,h+1}}(s_h^t,a_h^t), \eta_t=\bar\sigma_{k,h}^{-1}\vvalue_{t,h+1}(s_{h+1}^t)-\bar\sigma_{k,h}^{-1}[\PP_h\vvalue_{t,h+1}](s_{h}^t,a_{h}^t).$ For $\xb_t,\eta_t$, we have following property
\begin{align}
    \|\xb_t\|_2&=\big\|\bar\sigma_{k,h}^{-1}\bphi_{\vvalue_{t,h+1}}(s_h^t,a_h^t)\big\|_2\leq \bar\sigma_{k,h}^{-1}\max_{s'}\vvalue_{t,h+1}(s')\leq \sqrt{d}\notag\\
    \EE[\eta_t|\mathcal{F}_t]&=0,|\eta_t|=\big|\bar\sigma_{k,h}^{-1}\vvalue_{t,h+1}(s_{h+1}^t)-\bar\sigma_{k,h}^{-1}[\PP_h\vvalue_{t,h+1}](s_{h}^t,a_{h}^t)\big|\leq \sqrt{d}.\notag
\end{align}
 With probability at least $1-2\delta/H$, for all $t\in[K]$, we have 
\begin{align}
    \EE[\eta^2_t|\mathcal{F}_t]&= \bar\sigma_{t,h}^{-1}[\VV_h\vvalue_{t,h+1}](s_h^t,a_h^t),\leq \bar\sigma_{t,h}^{-1}\big([\bar\VV_h\vvalue_{t,h+1}](s_h^t,a_h^t)+E_{t,h}\big)\leq 1,
\end{align}
where the first inequality holds due to \eqref{eq:estimate-variance-UCB} and the second inequality holds due to the definition of $\bar\sigma_{t,h}^{-1}$. Therefore, with probability at least $1-3\delta/H$, for all $k\in[K]$, we have
\begin{align}
    \big\|\hat{\btheta}_{k,h}-{\btheta_h}\big\|_{\hat{\bSigma}_{k,h}}\leq 8\sqrt{d\log\big(1+kH^4/(d\lambda)\big)\log(4k^2H/\delta)}+4\sqrt{d}\log(4k^2H/\delta)+\sqrt{\lambda}B=\hat{\beta}_k.\notag
\end{align}
Taking union bound for all $h\in[H]$, we finish the proof.
\end{proof}

\subsection{Proof of Lemma \ref{LEMMA:TELESCOPE-SUM}}

\begin{proof}[Proof of Lemma \ref{LEMMA:TELESCOPE-SUM}]
For each $h\in[H]$ and $s\in\cS$, we have
\begin{align}
    \vvalue_{k,h}^{*}(s)-\vvalue_{k,h}(s)
    &=\EE_{a\sim \pi_h^*(\cdot|s)}\big[\qvalue_{k,h}^*(s,a)\big]-\EE_{a\sim \pi_h^k(\cdot|s)}\big[\qvalue_{k,h}(s,a)\big]\notag\\
    &=\underbrace{\EE_{a\sim \pi_h^*(\cdot|s)}\big[\qvalue_{k,h}^*(s,a)\big]-\EE_{a\sim \pi_h^*(\cdot|s)}\big[\qvalue_{k,h}(s,a)\big]}_{I}\notag\\
    &\qquad+\EE_{a\sim \pi_h^*(\cdot|s)}\big[\qvalue_{k,h}(s,a)\big]-\EE_{a\sim \pi_h^k(\cdot|s)}\big[\qvalue_{k,h}(s,a)\big]\label{eq:1}.
\end{align}
For the term $I$, we have
\begin{align}
    I&=\EE_{a\sim \pi_h^*(\cdot|s)}\big[\qvalue_{k,h}^*(s,a)-\qvalue_{k,h}(s,a)\big]\notag\\
    &\leq \EE_{a\sim \pi_h^*(\cdot|s)}\Big[[\PP_h(\vvalue^*_{k,h+1}-\vvalue_{k,h+1})\big](s,a)\Big]\notag\\
    &=\EE_{a\sim \pi_h^*(\cdot|s),s'\sim \PP_h(\cdot|s,a)}\big[\vvalue^*_{k,h+1}(s')-\vvalue_{k,h+1}(s')\big],\label{eq:6}
\end{align}
where the inequality holds due to Lemma \ref{LEMMA:TRANSITION1}.
Recursively using \eqref{eq:6} with all $h\in[H]$, for all $k\in[K]$, we have
\begin{align}
\vvalue_{k,1}^{*}(s_1^k)-\vvalue_{k,1}(s_1^k)\leq \EE\bigg[\sum_{h=1}^H\Big\{\EE_{a\sim \pi^*_{h}(\cdot|s_h)}\big[\qvalue_{k,h}(s_h, a)\big]-\EE_{a\sim \pi^k_{h}(\cdot|s_h)}\big[\qvalue_{k,h}(s_h, a)\big]\Big\}\big|s_1=s_1^k\bigg],\notag
\end{align}
where $a_{h} \sim \pi^*_{h}(\cdot|s_{h}),s_{h+1} \sim \PP_h(\cdot| s_{h}, a_{h})$.
Therefore, we finish the proof.
\end{proof}
\subsection{Proof of Lemma \ref{LEMMA:ONE-STEP-DESCENT}}
\begin{proof}[Proof of Lemma \ref{LEMMA:ONE-STEP-DESCENT}]
By the update rule of the policy $\pi_h^k$, 
for all $k\in[K]$,$h\in[H]$, $s\in\cS$, we have
\begin{align}
    \exp\big(\alpha\qvalue_{k,h}(s,a)\big)&=\frac{\pi_h^{k}(a|s) \exp\big\{\alpha\qvalue_{h,k}(a|s)\big\}}{\pi_h^{k}(a|s)}=\frac{\rho\pi_h^{k+1}(a|s)}{\pi_h^{k}(a|s)},\label{eq:31}
\end{align}
where $\rho=\sum_{a\in \cA}\pi_h^{k}(a|s) \exp\big\{\alpha\qvalue_{h,k}(a|s)\big\}$ is fixed for all action $a$.
Thus, we have
\begin{align}
    &\sum_{a\in \cA} \alpha\qvalue_{k,h}(s,a)\big(\pi^*_{h}(a|s)-\pi^{k+1}_{h}(a|s)\big)\notag\\
    &=\sum_{a\in \cA} \big(\log \rho+\log \pi_h^{k+1}(a|s) -\log \pi_h^{k}(a|s) \big)\big(\pi^*_{h}(a|s)-\pi^{k+1}_{h}(a|s)\big)\notag\\
    &=\sum_{a\in \cA} \pi^*_{h}(a|s)\big(\log \pi_h^{*}(a|s) -\log \pi_h^{k}(a|s) \big)\notag\\
    &= \sum_{a\in \cA} \pi^*_{h}(a|s)\big(\log \pi_h^{*}(a|s) -\log \pi_h^{k+1}(a|s) \big) -\sum_{a\in \cA}\pi^*_{h}(a|s) \big(\log \pi_h^{k+1}(a|s) -\log \pi_h^{k}(a|s) \big)\notag\\
    &\qquad -\sum_{a\in \cA} \pi^{k+1}_{h}(a|s)\big(\log \pi_h^{k+1}(a|s) -\log \pi_h^{k}(a|s) \big)\notag\\
    &=D_{KL}\big(\pi_h^*(\cdot|s)\|\pi_h^{k+1}(\cdot|s)\big)-D_{KL}\big(\pi_h^*(\cdot|s)\|\pi_h^{k}(\cdot|s)\big)-D_{KL}\big(\pi_h^{k+1}(\cdot|s)\|\pi_h^{k}(\cdot|s)\big),\label{eq:32}
\end{align}
where the first equation holds due to \eqref{eq:31} and the second equation holds due to $\sum_{a\in \cA} \big(\pi^*_{h}(a|s)-\pi^{k+1}_{h}(a|s)\big) =0$.
Therefore, we have
\begin{align}
    &\EE_{a\sim \pi^*_{h}(\cdot|s_h)}\big[\qvalue_{k,h}(s, a)\big]-\EE_{a\sim \pi^k_{h}(\cdot|s)}\big[\qvalue_{k,h}(s, a)\big]\notag\\
    &=\sum_{a\in \cA} \qvalue_{k,h}(s,a)\big(\pi^*_{h}(a|s)-\pi^k_{h}(a|s)\big)\notag\\
    &=\sum_{a\in \cA} \qvalue_{k,h}(s,a)\big(\pi^*_{h}(a|s)-\pi^{k+1}_{h}(a|s)\big)+\sum_{a\in \cA} \qvalue_{k,h}(s,a)\big(\pi^{k+1}_{h}(a|s)-\pi^k_{h}(a|s)\big)\notag\\
    &\leq \sum_{a\in \cA} \qvalue_{k,h}(s,a)\big(\pi^*_{h}(a|s)-\pi^{k+1}_{h}(a|s)\big)+ H\big\|\pi^{k+1}_{h}(\cdot|s)-\pi^k_{h}(\cdot|s)\big\|_1\notag\\
    &=\alpha^{-1}\Big(D_{KL}\big(\pi_h^*(\cdot|s)\|\pi_h^{k+1}(\cdot|s)\big)-D_{KL}\big(\pi_h^*(\cdot|s)\|\pi_h^{k}(\cdot|s)\big)-D_{KL}\big(\pi_h^{k+1}(\cdot|s)\|\pi_h^{k}(\cdot|s)\big)\Big)\notag\\
    &\qquad +H\big\|\pi^{k+1}_{h}(\cdot|s)-\pi^k_{h}(\cdot|s)\big\|_1\notag\\
    &\leq \alpha^{-1}\Big(D_{KL}\big(\pi_h^*(\cdot|s)\|\pi_h^{k+1}(\cdot|s)\big)-D_{KL}\big(\pi_h^*(\cdot|s)\|\pi_h^{k}(\cdot|s)\big)\Big)\notag\\
    &\qquad +H\big\|\pi^{k+1}_{h}(\cdot|s)-\pi^k_{h}(\cdot|s)\big\|_1-\frac{\big\|\pi^{k+1}_{h}(\cdot|s)-\pi^k_{h}(\cdot|s)\big\|_1^2}{2\alpha}\notag\\
    &\leq \frac{\alpha H^2}{2}+\alpha^{-1}\Big(D_{KL}\big(\pi_h^*(\cdot|s_h)\|\pi_h^k(\cdot|s_h)\big)-D_{KL}\big(\pi_h^*(\cdot|s_h)\|\pi_h^{k+1}(\cdot|s_h)\big)\Big),
\end{align}
where the first inequality holds due to the fact that $0\leq \qvalue_{k,h}^{\pi^k}(s,a)\leq \qvalue_{k,h}(s,a)\leq H$, the second inequality holds due to Pinsker’s inequality and the last inequality holds due to the fact that $ax-bx^2\leq a^2/4b$. Therefore, we finish the proof.

\end{proof}
\subsection{Proof of Lemma \ref{LEMMA:TRANSITION}}
\begin{proof}[Proof of Lemma \ref{LEMMA:TRANSITION}]
\begin{align}
    &\qvalue_{k,h}(s_h^k,a_h^k)-\qvalue^{\pi^k}_{k,h}(s_h^k,a_h^k)\notag
    \\&=\Big[\reward_h^k(s_h^k,a_h^k)+\big\la\hat{\btheta}_{k,h},\bphi_{\vvalue_{k,h+1}}(s_h^k,a_h^k)\big\ra+\hat{\beta}_k\big\|\hat{\bSigma}_{k,h}^{-1/2}\bphi_{\vvalue_{k,h+1}}(s_h^k,a_h^k)\big\|_2\Big]_{[0,H-h+1]}\notag\\
    &\qquad -  \reward_h^k(s_h^k,a_h^k) - [\PP_h\vvalue_{k,h}^{\pi^k}](s_h^k,a_h^k)\notag\\
    &\leq \Big|\big\la\hat{\btheta}_{k,h},\bphi_{\vvalue_{k,h+1}}(s_h^k,a_h^k)\big\ra+\hat{\beta}_k\big\|\hat{\bSigma}_{k,h}^{-1/2}\bphi_{\vvalue_{k,h+1}}(s_h^k,a_h^k)\big\|_2 \Big|-[\PP_h\vvalue_{k,h}^{\pi^k}](s_h^k,a_h^k)\notag\\
    &\leq \big[\PP_h \vvalue_{k,h+1}\big](s_h^k,a_h^k)+\Big|\big\la\hat{\btheta}_{k,h}-\btheta_h,\bphi_{\vvalue_{k,h+1}}(s_h^k,a_h^k)\big\ra\Big|\notag\\
    &\qquad +\hat{\beta}_k\big\|\hat{\bSigma}_{k,h}^{-1/2}\bphi_{\vvalue_{k,h+1}}(s_h^k,a_h^k)\big\|_2-\big[\PP_h\vvalue_{k,h+1}^{\pi^k}\big](s_h^k,a_h^k)\notag\\
    &\leq \big[\PP_h(\vvalue_{k,h+1}-\vvalue_{k,h+1}^{\pi^k})\big](s_h^k,a_h^k) +2\hat{\beta}_k\big\|\hat{\bSigma}_{k,h}^{-1/2}\bphi_{\vvalue_{k,h+1}}(s_h^k,a_h^k)\big\|_2\notag\\
     &= \big[\PP_h(\vvalue_{k,h+1}-\vvalue_{k,h+1}^{\pi^k})\big](s_h^k,a_h^k) +2\hat{\beta}_k\bar\sigma_{k,h}\big\|\hat{\bSigma}_{k,h}^{-1/2}\bphi_{\vvalue_{k,h+1}}(s_h^k,a_h^k)/\bar\sigma_{k,h}\big\|_2,\label{eq:41}
\end{align}
where the first inequality holds due to the fact that $x_{[0,z]}-y\leq |x-y|$ when $y\ge 0$, the second inequality holds due to the fact that $|x+y+z|\leq |x|+|y|+|z|$ and the third inequality holds due to event $\cE$. Furthermore, we have
\begin{align}
    \qvalue_{k,h}(s_h^k,a_h^k)-\qvalue^{\pi^k}_{k,h}(s_h^k,a_h^k)\leq H-\qvalue^{\pi^k}_{k,h}\leq H\leq 2\hat{\beta}_k\bar\sigma_{k,h},\label{eq:42}
\end{align}
where the first inequality holds due to $\qvalue_{k,h}(s_h^k,a_h^k)\leq H$, the second inequality holds due to $\qvalue^{\pi^k}_{k,h}(s_h^k,a_h^k)\ge 0$ and the last inequality holds due to $2\hat{\beta}_k\bar\sigma_{k,h}\ge \sqrt{d}H/\sqrt{d}=H$. Combined \eqref{eq:41} and \eqref{eq:42}, we have
\begin{align}
    \qvalue_{k,h}(s_h^k,a_h^k)-\qvalue^{\pi^k}_{k,h}(s_h^k,a_h^k)&\leq \big[\PP_h(\vvalue_{k,h+1}-\vvalue_{k,h+1}^{\pi^k})\big](s_h^k,a_h^k) \notag\\
    &\qquad+2\hat{\beta}_k \bar\sigma_{k,h}\min \Big\{\big\|\hat{\bSigma}_{k,h}^{-1/2}\bphi_{\vvalue_{k,h+1}}(s_h^k,a_h^k)/\bar\sigma_{k,h}\big\|_2,1\Big\}.\notag
\end{align}
Therefore, we finish the proof.
\end{proof}
\subsection{Proof of Lemma \ref{LEMMA:TRANSITION2}}
\begin{proof}[Proof of Lemma \ref{LEMMA:TRANSITION2}]
\begin{align}
    &\vvalue_{k,h}(s_h^k)-\vvalue_{k,h}^{\pi^k}(s_h^k) \notag\\
   &=\EE_{a\sim \pi_h^k(\cdot|s_h^k)}\big[\qvalue_{k,h}(s_h^k,a)-\qvalue_{k,h}^{\pi^k}(s_h^k,a)\big]\notag\\
   &=\EE_{a\sim \pi_h^k(\cdot|s_h^k)}\big[\qvalue_{k,h}(s_h^k,a)-\qvalue_{k,h}^{\pi^k}(s_h^k,a)\big]-\big(\qvalue_{k,h}(s_h^k,a_h^k)-\qvalue_{k,h}^{\pi^k}(s_h^k,a_h^k\big)+\qvalue_{k,h}(s_h^k,a_h^k)-\qvalue^{\pi^k}_{k,h}(s_h^k,a_h^k)\notag\\
   &\leq \EE_{a\sim \pi_h^k(\cdot|s_h^k)}\big[\qvalue_{k,h}(s_h^k,a)-\qvalue_{k,h}^{\pi^k}(s_h^k,a)\big]-\big(\qvalue_{k,h}(s_h^k,a_h^k)-\qvalue_{k,h}^{\pi^k}(s_h^k,a_h^k\big)+\big[\PP_h(\vvalue_{k,h+1}-\vvalue_{k,h+1}^{\pi^k})\big](s_h^k,a_h^k)\notag\\
   &\qquad +2\hat{\beta}_k \bar\sigma_{k,h}\min \Big\{\big\|\hat{\bSigma}_{k,h}^{-1/2}\bphi_{\vvalue_{k,h+1}}(s_h^k,a_h^k)/\bar\sigma_{k,h}\big\|_2,1\Big\}\notag\\
   &= \EE_{a\sim \pi_h^k(\cdot|s_h^k)}\big[\qvalue_{k,h}(s_h^k,a)-\qvalue_{k,h}^{\pi^k}(s_h^k,a)\big]-\big(\qvalue_{k,h}(s_h^k,a_h^k)-\qvalue_{k,h}^{\pi^k}(s_h^k,a_h^k\big)\notag\\
   &\qquad+\vvalue_{k,h+1}(s_{h+1}^{k})-\vvalue_{k,h+1}^{\pi^k}(s_{h+1}^{k})+\big[\PP_h(\vvalue_{k,h+1}-\vvalue_{k,h+1}^{\pi^k})\big](s_h^k,a_h^k)-\big(\vvalue_{k,h+1}(s_{h+1}^{k})-\vvalue_{k,h+1}^{\pi^k}(s_{h+1}^{k})\big) \notag\\
   &\qquad +2\hat{\beta}_k \bar\sigma_{k,h}\min \Big\{\big\|\hat{\bSigma}_{k,h}^{-1/2}\bphi_{\vvalue_{k,h+1}}(s_h^k,a_h^k)/\bar\sigma_{k,h}\big\|_2,1\Big\},\label{eq:44}
\end{align}
where the inequality holds due to Lemma \ref{LEMMA:TRANSITION}. 
Furthermore, on the event $\cE$ and $\cE_1$, for all $h\in[H]$, we have
\begin{align}
   &\sum_{k=1}^K\big(\vvalue_{k,h}(s_h^k)-\vvalue_{k,h}^{\pi^k}(s_h^k)\big)\notag\notag\\
   &\leq \sum_{k=1}^K\sum_{h'=h}^{H}2\hat{\beta}_k \bar\sigma_{k,h}\min \Big\{\big\|\hat{\bSigma}_{k,h}^{-1/2}\bphi_{\vvalue_{k,h+1}}(s_h^k,a_h^k)/\bar\sigma_{k,h}\big\|_2,1\Big\}\notag\\
   &\qquad +\sum_{k=1}^K\sum_{h'=h}^{H}\Big(\EE_{a\sim \pi_h^k(\cdot|s_h^k)}\big[\qvalue_{k,h}(s_h^k,a)-\qvalue_{k,h}^{\pi^k}(s_h^k,a)\big]-\big(\qvalue_{k,h}(s_h^k,a_h^k)-\qvalue_{k,h}^{\pi^k}(s_h^k,a_h^k\big)\Big)\notag\\
   &\qquad +\sum_{k=1}^K\sum_{h'=h}^{H}\Big(\big[\PP_h(\vvalue_{k,h+1}-\vvalue_{k,h+1}^{\pi^k})\big](s_h^k,a_h^k)-\big(\vvalue_{k,h+1}(s_{h+1}^{k})-\vvalue_{k,h+1}^{\pi^k}(s_{h+1}^{k})\big)\Big)\notag\\
   &\leq \sum_{k=1}^K\sum_{h'=h}^{H}2\hat{\beta}_k \bar\sigma_{k,h}\min \Big\{\big\|\hat{\bSigma}_{k,h}^{-1/2}\bphi_{\vvalue_{k,h+1}}(s_h^k,a_h^k)/\bar\sigma_{k,h}\big\|_2,1\Big\}+4H\sqrt{T\log(H/\delta)}\notag\\
   &\leq 2\hat{\beta}_K\sum_{k=1}^K\sum_{h'=h}^{H} \bar\sigma_{k,h}\min \Big\{\big\|\hat{\bSigma}_{k,h}^{-1/2}\bphi_{\vvalue_{k,h+1}}(s_h^k,a_h^k)/\bar\sigma_{k,h}\big\|_2,1\Big\}+4H\sqrt{T\log(H/\delta)}\notag\\
   &\leq 2\hat{\beta}_K\sqrt{\sum_{k=1}^K\sum_{h=1}^H\bar\sigma_{k,h}^2}\sqrt{\sum_{k=1}^K\sum_{h=1}^H\min \Big\{\big\|\hat{\bSigma}_{k,h}^{-1/2}\bphi_{\vvalue_{k,h+1}}(s_h^k,a_h^k)/\bar\sigma_{k,h}\big\|_2,1\Big\}}+4H\sqrt{T\log(H/\delta)}\notag\\
   &\leq 2\hat{\beta}_K\sqrt{\sum_{k=1}^K\sum_{h=1}^H\bar\sigma_{k,h}^2}\sqrt{2Hd\log(1+K/\lambda)}+4H\sqrt{T\log(H/\delta)},\label{eq:45}
\end{align}
where the first inequality holds by taking the summation of \eqref{eq:44} for $k\in[K]$ and $h\leq h'\leq H$, the second inequality holds due to the definition of event $\cE_1$, the third inequality holds due to $\hat{\beta}_k\leq \hat{\beta}_K$, the fourth inequality holds due to Cauchy-Schwarz inequality and the last inequality holds due to Lemma \ref{lemma: sum-UCB}. Furthermore, taking the summation of \eqref{eq:45}, we have

\begin{align}
    &\sum_{k=1}^K\sum_{h=1}^H \big[\PP_h(\vvalue_{k,h+1}-\vvalue_{k,h+1}^{\pi^k})\big](s_h^k,a_h^k)\notag\\
    &=\sum_{k=1}^K\sum_{h=1}^H\big(\vvalue_{k,h+1}(s_{h+1}^{k})-\vvalue_{k,h+1}^{\pi^k}(s_{h+1}^{k})\big)\notag\\
    &\qquad+\sum_{k=1}^K\sum_{h=1}^{H}\Big(\big[\PP_h(\vvalue_{k,h+1}-\vvalue_{k,h+1}^{\pi^k})\big](s_h^k,a_h^k)-\big(\vvalue_{k,h+1}(s_{h+1}^{k})-\vvalue_{k,h+1}^{\pi^k}(s_{h+1}^{k})\big)\Big)\notag\\
    &\leq \sum_{k=1}^K\sum_{h=1}^H\big(\vvalue_{k,h+1}(s_{h+1}^{k})-\vvalue_{k,h+1}^{\pi^k}(s_{h+1}^{k})\big)+2H\sqrt{2T\log(1/\delta)}\notag\\
    &\leq 2H\hat{\beta}_K\sqrt{\sum_{k=1}^K\sum_{h=1}^H\bar\sigma_{k,h}^2}\sqrt{2Hd\log(1+K/\lambda)}+4H^2\sqrt{T\log(H/\delta)},\notag
\end{align}
where the first inequality holds due to the definition of event $\cE_2$ and the last inequality holds due \eqref{eq:45}. Therefore, we finish the proof.
\end{proof}

\subsection{Proof of Lemma \ref{LEMMA: TOTAL-ESTIMATE-VARIANCE}}

\begin{proof}[Proof of Lemma \ref{LEMMA: TOTAL-ESTIMATE-VARIANCE}]
On the event $\cE$, by Lemma \ref{LEMMA: CONCENTRATE}, for all $k\in[K],h\in[H]$, we have 
\begin{align}
    [\bar{\VV}_{k,h}\vvalue_{k,h+1}](s_h^k,a_h^k)+E_{k,h}\ge [{\VV}_{h}\vvalue_{k,h+1}](s_h^k,a_h^k)\ge 0.\notag
\end{align} 
Therefore, we have
\begin{align}
    \sum_{k=1}^K\sum_{h=1}^H\bar\sigma_{k,h}^{2} &=\sum_{k=1}^K\sum_{h=1}^H\max\big\{H^2/d,         [\bar{\VV}_{k,h}\vvalue_{k,h+1}](s_h^k,a_h^k)+E_{k,h}  \big\}\notag\\
    &\leq\sum_{k=1}^K\sum_{h=1}^H \frac{H^2}{d}+\sum_{k=1}^K\sum_{h=1}^H [\bar{\VV}_{k,h}\vvalue_{k,h+1}](s_h^k,a_h^k)+\sum_{k=1}^K\sum_{h=1}^HE_{k,h} \notag\\
    &=\frac{H^2T}{d}+\underbrace{\sum_{k=1}^K\sum_{h=1}^H \big([\bar{\VV}_{k,h}\vvalue_{k,h+1}](s_h^k,a_h^k)-[{\VV}_{h}\vvalue_{k,h+1}](s_h^k,a_h^k)\big)}_{I_1}+\underbrace{\sum_{k=1}^K\sum_{h=1}^HE_{k,h}}_{I_2}\notag\\
    &\qquad +\underbrace{\sum_{k=1}^K\sum_{h=1}^H \big([{\VV}_{h}\vvalue_{k,h+1}](s_h^k,a_h^k)-[{\VV}_{h}\vvalue^{\pi^k}_{k,h+1}](s_h^k,a_h^k)\big)}_{I_3}+\underbrace{\sum_{k=1}^K\sum_{h=1}^H [{\VV}_{h}\vvalue^{\pi^k}_{k,h+1}](s_h^k,a_h^k)}_{I_4},\label{eq:61}
\end{align}
where the inequality holds due to the fact that $\max\{a,b\}\leq a+b$, when $a,b\ge 0$.
For the term $I_1$, we have
\begin{align}
    I_1&=\sum_{k=1}^K\sum_{h=1}^H \big([\bar{\VV}_{k,h}\vvalue_{k,h+1}](s_h^k,a_h^k)-[{\VV}_{h}\vvalue_{k,h+1}](s_h^k,a_h^k)\big)\leq \sum_{k=1}^K\sum_{h=1}^H E_{k,h}= I_2,\label{eq:62}
\end{align}
where the inequality holds due to the definition of event $\cE$. For the term $I_2$, we have
\begin{align}
    I_2&=\sum_{k=1}^K\sum_{h=1}^H E_{k,h}\notag\\
    &=\sum_{k=1}^K\sum_{h=1}^H  \min \Big\{\tilde{\beta}_k\big\|\tilde{\bSigma}_{k,h}^{-1/2}\bphi_{\vvalue^2_{k,h+1}}(s_h^k,a_h^k)\big\|_2,H^2\Big\}+\sum_{k=1}^K\sum_{h=1}^H\min \Big\{2H\bar{\beta}_k\big\|\hat{\bSigma}_{k,h}^{-1/2}\bphi_{\vvalue_{k,h+1}}(s_h^k,a_h^k)\big\|_2,H^2\Big\}\notag\\
    &\leq \tilde{\beta}_{K}\sum_{k=1}^K\sum_{h=1}^H\min \Big\{\big\|\tilde{\bSigma}_{k,h}^{-1/2}\bphi_{\vvalue^2_{k,h+1}}(s_h^k,a_h^k)\big\|_2,1\Big\}\notag\\
    &\qquad +2H\bar{\beta}_K\sum_{k=1}^K\sum_{h=1}^H \bar{\sigma}_{k,h}\min \Big\{\big\|\hat{\bSigma}_{k,h}^{-1/2}\bphi_{\vvalue_{k,h+1}}(s_h^k,a_h^k)\big\|_2,1\Big\}\notag\\
    &\leq \tilde{\beta}_{K}\sqrt{T}\sqrt{\sum_{k=1}^K\sum_{h=1}^H\min \Big\{\big\|\tilde{\bSigma}_{k,h}^{-1/2}\bphi_{\vvalue^2_{k,h+1}}(s_h^k,a_h^k)\big\|^2_2,1\Big\}}\notag\\
    &\qquad +2\sqrt{3}H^2\bar{\beta}_K\sqrt{T}\sqrt{\sum_{k=1}^K\sum_{h=1}^H \min \Big\{\big\|\hat{\bSigma}_{k,h}^{-1/2}\bphi_{\vvalue_{k,h+1}}(s_h^k,a_h^k)\big\|_2,1\Big\}}\notag\\
    &\leq  \tilde{\beta}_{K}\sqrt{T} \sqrt{2dH \log\big(1+kH^4/(d\lambda)\big)}+2H^2\bar{\beta}_K\sqrt{3T} \sqrt{2dH\log(1+K/\lambda)},\label{eq:63}
\end{align}
where the first inequality holds on due to $ \tilde{\beta}_{K}\ge  \tilde{\beta}_{k}\ge H^2$ and $\bar{\beta}_K \bar{\sigma}_{k,h}\ge \bar{\beta}_k \bar{\sigma}_{k,h}\ge H$, the second inequality holds due to Cauchy-Schwarz inequality with the fact that $\bar{\sigma}_{k,h}\leq \max\{H^2/d,H^2+2H^2\}\leq 3H^2$ and the last inequality holds due to Lemma \ref{lemma: sum-UCB}. 

On the event $\cE\cap \cE_1\cap\cE_2$, for the term $I_3$, we have  
\begin{align}
    I_3&=\sum_{k=1}^K\sum_{h=1}^H \big([{\VV}_{h}\vvalue_{k,h+1}](s_h^k,a_h^k)-[{\VV}_{h}\vvalue^{\pi^k}_{k,h+1}](s_h^k,a_h^k)\big)\notag\\
    &=\sum_{k=1}^K\sum_{h=1}^H \Big([\PP_h \vvalue_{k,h+1}^2](s_h^k,a_h^k)-\big([\PP_h \vvalue_{k,h+1}](s_h^k,a_h^k)\big)^2
    -[\PP_h (\vvalue^{\pi^k}_{k,h+1})^2](s_h^k,a_h^k)+\big([\PP_h \vvalue^{\pi^k}_{k,h+1}](s_h^k,a_h^k)\big)^2\Big)\notag\\
    &\leq \sum_{k=1}^K\sum_{h=1}^H\big([\PP_h \vvalue_{k,h+1}^2](s_h^k,a_h^k)-[\PP_h (\vvalue^{\pi^k}_{k,h+1})^2](s_h^k,a_h^k)\big)\notag\\
    &\leq 2H \sum_{k=1}^K\sum_{h=1}^H\big([\PP_h \vvalue_{k,h+1}](s_h^k,a_h^k)-[\PP_h \vvalue^{\pi^k}_{k,h+1}](s_h^k,a_h^k)\big)\notag\\
    &\leq 4H^2\hat{\beta}_K\sqrt{\sum_{k=1}^K\sum_{h=1}^H\bar\sigma_{k,h}^2}\sqrt{2Hd\log(1+K/\lambda)}+8H^3\sqrt{T\log (H/\delta)},\label{eq:64}
\end{align}
where the first inequality holds due to the fact that $\vvalue^{\pi^k}_{k,h+1}(s')\leq \vvalue_{k,h+1}(s')$, the second inequality holds due to $0\leq \vvalue_{k,h+1}(s'),\vvalue^{\pi^k}_{k,h+1}(s') \leq H$ and the last inequality holds due to Lemma \ref{LEMMA:TRANSITION2}.

On the event $\cE_3$, for the term $I_4$, we have
\begin{align}
    I_4=\sum_{k=1}^K\sum_{h=1}^H [{\VV}_{h}\vvalue^{\pi^k}_{k,h+1}](s_h^k,a_h^k)\leq 3\big(HT+H^3\log(1/\delta)\big).\label{eq:65}
\end{align}
Substituting \eqref{eq:62}, \eqref{eq:63}, \eqref{eq:64} and \eqref{eq:65} into \eqref{eq:61}, we have
\begin{align}
    \sum_{k=1}^K\sum_{h=1}^H\bar\sigma_{k,h}^{-2}&\leq H^2T/d+ 3\big(HT+H^3\log(1/\delta)\big)\notag\\
    &\qquad + 2\tilde{\beta}_{K}\sqrt{T} \sqrt{2dH \log\big(1+KH^4/(d\lambda)\big)}+4H^2\bar{\beta}_K\sqrt{3T} \sqrt{2dH\log(1+K/\lambda)}\notag\\
    &\qquad +4H^2\hat{\beta}_K\sqrt{\sum_{k=1}^K\sum_{h=1}^H\bar\sigma_{k,h}^2}\sqrt{2Hd\log(1+K/\lambda)}+8H^3\sqrt{T\log (H/\delta)},\notag
\end{align}
where 
\begin{align}
    \bar{\beta}_K&=8d\sqrt{\log(1+K/\lambda)\log(4K^2H/\delta)}+4\sqrt{d}\log(4K^2H/\delta)+\sqrt{\lambda}B,\notag\\
   \hat{\beta_K}&= 8\sqrt{d\log(1+K/\lambda)\log(4K^2H/\delta)}+4\sqrt{d}\log(4K^2H/\delta)+\sqrt{\lambda}B,\notag\\
   \tilde{\beta}_K&=8H^2\sqrt{d\log\big(1+KH^4/(d\lambda)\big)\log(4K^2H/\delta)}+4H^2\log(4K^2H/\delta)+\sqrt{\lambda}B.\notag
\end{align}
Therefore, by the fact that $x\leq a\sqrt{x}+b$ implies $x\leq a^2+2b$, we have
\begin{align}
    \sum_{k=1}^K\sum_{h=1}^H\bar\sigma_{k,h}^{-2}&\leq 2H^2T/d+ 6\big(HT+H^3\log(1/\delta)\big)\notag\\
    &\qquad + 4\tilde{\beta}_{K}\sqrt{T} \sqrt{2dH \log\big(1+KH^4/(d\lambda)\big)}+8H^2\bar{\beta}_K\sqrt{3T} \sqrt{2dH\log(1+K/\lambda)}\notag\\
    &\qquad +32H^5d(\hat{\beta}_K)^2 \log(1+K/\lambda)+16H^3\sqrt{T\log (H/\delta)}\notag\\
    &\leq 2H^2T/d+6\big(HT+H^3\log(1/\delta)\big)+330\sqrt{d^3H^5T}\log(4K^2H/\delta)\log(1+KH^4/\lambda)\notag\\
    &\qquad+2048H^5d^2 \log^2(4K^2H/\delta)\log^2(1+K/\lambda)+16H^3\sqrt{T\log (H/\delta)}\notag\\
    &\leq 2HT/d+179HT+165d^3H^4 \log^2(4K^2H/\delta)\log^2(1+KH^4/\lambda)\notag\\
    &\qquad +2062H^5d^2 \log^2(4K^2H/\delta)\log^2(1+K/\lambda),\notag
\end{align}
where the second inequality holds due to the definition of parameter $\bar{\beta}_K,\hat{\beta_K},\tilde{\beta}_K$ with the fact that $\lambda=1/B^2\leq 1$ and the third inequality holds due to Young's inequality. Therefore, we finish the proof.
\end{proof}


\end{document}